\documentclass[]{article}

\usepackage[preprint]{neurips_2026}
\workshoptitle{}

\usepackage[]{authblk}
\usepackage{courier}
\usepackage{fix-cm}
\usepackage{hyperref}
\usepackage{xurl}

\usepackage{algorithm}
\usepackage{algpseudocode}

\usepackage{amsmath}
\usepackage{amsthm}
\usepackage{stix}
\newtheorem{theorem}{Theorem}

\usepackage{booktabs}
\newcommand{\lessismore}[1]{\makebox[2.75em][l]{\hspace{0.10em}\makebox[0.55em][c]{{\color{teal}\tiny$\downarrow$}}\makebox[2.00em][r]{{\color{teal}\tiny$#1\%$}}}}
\newcommand{\moreisless}[1]{\makebox[2.75em][l]{\hspace{0.10em}\makebox[0.50em][c]{{\color{purple}\tiny$\uparrow$}}\makebox[2.00em][r]{{\color{purple}\tiny$#1\%$}}}}

\usepackage{tikz}
\usetikzlibrary{arrows.meta}

\title{%
	Linear-Time Non-Local Codec Placement \\
	for Machine Learning Model Compression
}
\author{%
    Matteo Spallanzani \\
    Axelera AI \\
    Zürich, Switzerland \\
    \texttt{matteo.spallanzani@axelera.ai} \\
}

\begin{document}

\maketitle

\begin{abstract}
	On-device inference imposes constraints on model storage and operand traffic and thus benefits from compressed operand representations, which mandate the use of encoding and decoding operators.
	However, na\"{\i}ve placement of codec operators can cancel the operand traffic gains sought from compression.
	In this work, I present an algorithm to place codec operators into a source computational graph to make it compatible with target hardware.
	The edges of the graph are colored so that each induced subgraph is weakly connected and has a maximal node that both postdominates all subgraph inputs and dominates all subgraph outputs; this unique node \textemdash dubbed the \textit{absolute dominator} \textemdash is the natural location to observe the traffic flowing through the subgraph and learn parametric codec operators.
	The algorithm's time complexity is linear in the source graph's size.
	With respect to na\"{\i}ve local codec placement, the algorithm reduces modeled operand traffic by $24\%$--$44\%$ across four models, including $24\%$ on SigLIP and $34\%$ on SmolLM2.
\end{abstract}

\section{Rationale}
\label{section:rationale}

Machine learning (ML) model compression increases efficiency by representing operands (weights and features) with low-bitwidth data types.
Compressing source models developed for general-purpose hardware for deployment on target hardware featuring one or more accelerators requires using data-type-changing operators.
An example is the quantization of GPU-native models \textemdash where operands are represented as 32- or 16-bit floating-point (FP) numbers \textemdash into low-bitwidth data types \textemdash such as 8- or 4-bit integers (INT).

The data-type-changing operators are usually lossy encoder-decoder (\textbf{codec}) operators.
In regular programs, these new operators would break correctness.
ML models are different in that they can use parametric lossy codecs and statistical algorithms to mitigate application accuracy degradations.
These algorithms include distribution transformations~\citep{2023Xiao, 2024Ashkboos}, post-training quantization (PTQ)~\citep{2021Nagel}, and quantization-aware training (QAT)~\citep{2021Nagel}.
This codec abstraction encompasses the established case of element-wise linear quantization~\citep{2013Bengio, 2020Esser}, element-wise non-linear quantization~\citep{2017Zhou}, block-wise quantization~\citep{2023Rouhani}, and vector quantization~\citep{2026Zandieh}.

Codec placement can range in complexity from local to non-local.
The difference between local and non-local placement is best illustrated by an example.
Consider a two-branch residual block with a matrix-matrix multiplication on each branch: a local rule quantizes the weight and feature inputs of each operator independently; a non-local rule quantizes the shared block input only once while still quantizing the weights independently.
This example shows that to achieve operand traffic savings one must decide not only which operators use compressed representations, but also where encodings and decodings occur across operator boundaries.
The decision whether to perform local or non-local placement informs the choice of the underlying representation:
\begin{itemize}
	\item \textit{ahead-of-time (AoT)} representations \textemdash also referred to as \textit{graph} representations \textemdash arrange computations into a static computational graph before they are executed;
	\item \textit{just-in-time (JiT)} representations \textemdash also known as \textit{eager} representations \textemdash track computations into a dynamic computational graph while they are executed.
\end{itemize}
Non-local placement requires that eager representations be traced into graph representations.

Existing ML model compression frameworks~\citep{2019Roesch, 2019Xilinx, 2021Lattner, 2025Or, 2026NVIDIA, 2026Nachin} implement a variety of placement strategies.
Na\"{\i}ve strategies use local rules to achieve local placement.
More advanced strategies \textemdash such as TensorRT's Q/DQ propagation \textemdash use local commutation rules to grow low-precision subgraphs after learning; these strategies can reconcile codecs across operator boundaries only when they have equal parameters, and thus rely on codec parameter sharing specifications (i.e., non-local rules) defined ahead of learning to prevent learning-time parameter divergence.
Other advanced strategies use hardware-specific subgraph patterns (i.e., non-local rules) to achieve non-local placement.

This work presents an algorithm that 1) derives non-local codec placement ahead of learning from per-operator local commutation rules, 2) guarantees structural correctness, and 3) has linear-time complexity.
The non-local rules in codec parameter sharing specifications or pattern catalogs specify codec intent by context, which might depend on both the model to be compressed and target hardware.
In contrast, the local commutation rules consumed by this algorithm specify codec intent per operator \textemdash for each operator, a commutation rule declares which of its ports a codec can commute through \textemdash and they can be reused across every context in which an operator appears.
The local commutation rules are specific to the codec class and parametrization; this constraint is shared by existing non-local rules.
The linear-time complexity makes this algorithm particularly suited to deployment toolchains, where graphs might be adapted repeatedly under different compression policies.

\section{The algorithm}
\label{section:algorithm}

\begin{figure}[t]
\centering
\resizebox{\textwidth}{!}{%
	\begin{tabular}{@{}c@{\hspace{2em}}c@{\hspace{2em}}c@{\hspace{2em}}c@{}}
		\begin{tikzpicture}
    \node[draw=none, node font=\LARGE] at (2,-1) {(a)};
    \draw[arrows=-Latex] (-1,7) -- (-1,5.5);
    \draw[arrows=-Latex] (3,7) -- (3,5.5);
    \draw[arrows=-Latex] (-1,4.5) -- (0.65,2.35);
    \draw[arrows=-Latex] (3,4.5) -- (1.35,2.35);
    \draw[arrows=-Latex] (1,1.5) -- (1,0);
    \draw[arrows=-Latex] (3,4.5) -- (4.65,2.35);
    \draw[arrows=-Latex] (5,1.5) -- (5,0);
    \draw[fill=white] (-1,5) circle (0.5cm);
    \node[draw=none] at (-1,5) {R};
    \draw[fill=white] (3,5) circle (0.5cm);
    \node[draw=none] at (3,5) {R};
    \draw[fill=white] (1,2) circle (0.5cm);
    \node[draw=none] at (1,2) {C};
    \draw[fill=white] (5,2) circle (0.5cm);
    \node[draw=none] at (5,2) {R};
\end{tikzpicture}  &
		\begin{tikzpicture}
    \node[draw=none, node font=\LARGE] at (2,-1) {(b)};
    \draw[arrows=-Latex] (-1,7) -- (-1,6.25);
    \draw[arrows=-Latex] (-1,5.75) -- (-1,5.5);
    \draw[arrows=-Latex] (-1,4.5) -- (-1,4.25);
    \draw[arrows=-Latex] (3,7) -- (3,6.25);
    \draw[arrows=-Latex] (3,5.75) -- (3,5.5);
    \draw[arrows=-Latex] (3,4.5) -- (3,4.25);
    \draw[arrows=-Latex] (-0.85,3.85) -- (-0.15,3.15);
    \draw[arrows=-Latex] (0.15,2.85) -- (0.65,2.35);
    \draw[arrows=-Latex] (2.85,3.85) -- (2.15,3.15);
    \draw[arrows=-Latex] (1.85,2.85) -- (1.35,2.35);
    \draw[arrows=-Latex] (1,1.5) -- (1,1.25);
    \draw[arrows=-Latex] (1,0.75) -- (1,0);
    \draw[arrows=-Latex] (3.15,3.85) -- (3.85,3.15);
    \draw[arrows=-Latex] (4.15,2.85) -- (4.65,2.35);
    \draw[arrows=-Latex] (5,1.5) -- (5,1.25);
    \draw[arrows=-Latex] (5,0.75) -- (5,0);
    \draw[fill=black!16] (-1,6) circle (0.25cm);
    \draw[fill=white] (-1,5) circle (0.5cm);
    \node[draw=none] at (-1,5) {R};
    \draw[fill=black!40] (-1,4) circle (0.25cm);
    \draw[fill=black!16] (3,6) circle (0.25cm);
    \draw[fill=white] (3,5) circle (0.5cm);
    \node[draw=none] at (3,5) {R};
    \draw[fill=black!40] (3,4) circle (0.25cm);
    \draw[fill=black!16] (0,3) circle (0.25cm);
    \draw[fill=black!16] (2,3) circle (0.25cm);
    \draw[fill=white] (1,2) circle (0.5cm);
    \node[draw=none] at (1,2) {C};
    \draw[fill=black!40] (1,1) circle (0.25cm);
    \draw[fill=black!16] (4,3) circle (0.25cm);
    \draw[fill=white] (5,2) circle (0.5cm);
    \node[draw=none] at (5,2) {R};
    \draw[fill=black!40] (5,1) circle (0.25cm);
\end{tikzpicture}  &
		\begin{tikzpicture}
    \node[draw=none, node font=\LARGE] at (2,-1) {(c)};
    \draw[arrows=-Latex, ultra thick, dashed] (-1,7) -- (-1,6.25);
    \draw[arrows=-Latex, red, ultra thick, densely dashed] (-1,5.75) -- (-1,5.5);
    \draw[arrows=-Latex, red, ultra thick, densely dashed] (-1,4.5) -- (-1,4.25);
    \draw[arrows=-Latex, ultra thick, dotted, black!80] (3,7) -- (3,6.25);
    \draw[arrows=-Latex, orange, ultra thick, densely dotted] (3,5.75) -- (3,5.5);
    \draw[arrows=-Latex, orange, ultra thick, densely dotted] (3,4.5) -- (3,4.25);
    \draw[arrows=-Latex, red, ultra thick, densely dashed] (-0.85,3.85) -- (-0.15,3.15);
    \draw[arrows=-Latex, purple, ultra thick] (0.15,2.85) -- (0.65,2.35);
    \draw[arrows=-Latex, orange, ultra thick, densely dotted] (2.85,3.85) -- (2.15,3.15);
    \draw[arrows=-Latex, purple, ultra thick] (1.85,2.85) -- (1.35,2.35);
    \draw[arrows=-Latex, purple, ultra thick] (1,1.5) -- (1,1.25);
    \draw[arrows=-Latex, purple, ultra thick] (1,0.75) -- (1,0);
    \draw[arrows=-Latex, orange, ultra thick, densely dotted] (3.15,3.85) -- (3.85,3.15);
    \draw[arrows=-Latex, ultra thick, teal, loosely dashed] (4.15,2.85) -- (4.65,2.35);
    \draw[arrows=-Latex, ultra thick, teal, loosely dashed] (5,1.5) -- (5,1.25);
    \draw[arrows=-Latex, ultra thick, teal, loosely dashed] (5,0.75) -- (5,0);
    \draw[fill=black!16] (-1,6) circle (0.25cm);
    \draw[fill=white] (-1,5) circle (0.5cm);
    \node[draw=none] at (-1,5) {R};
    \draw[fill=black!40] (-1,4) circle (0.25cm);
    \draw[fill=black!16] (3,6) circle (0.25cm);
    \draw[fill=white] (3,5) circle (0.5cm);
    \node[draw=none] at (3,5) {R};
    \draw[fill=orange!80] (3,4) circle (0.25cm);
    \draw[fill=red!80] (0,3) circle (0.25cm);
    \draw[fill=black!16] (2,3) circle (0.25cm);
    \draw[fill=white] (1,2) circle (0.5cm);
    \node[draw=none] at (1,2) {C};
    \draw[fill=purple!80] (1,1) circle (0.25cm);
    \draw[fill=black!16] (4,3) circle (0.25cm);
    \draw[fill=white] (5,2) circle (0.5cm);
    \node[draw=none] at (5,2) {R};
    \draw[fill=teal!80] (5,1) circle (0.25cm);
\end{tikzpicture}   &
		\begin{tikzpicture}
    \node[draw=none, node font=\LARGE] at (2,-1) {(d)};
    \draw[arrows=-Latex, very thick, dashed] (-1,7) -- (-1,6.25);
    \draw[arrows=-Latex, ultra thick, purple] (-1,5.75) -- (-1,5.5);
    \draw[arrows=-Latex, ultra thick, purple] (-1,4.5) -- (-1,4.25);
    \draw[arrows=-Latex, very thick, dotted, black!80] (3,7) -- (3,6.25);
    \draw[arrows=-Latex, orange, very thick, densely dotted] (3,5.75) -- (3,5.5);
    \draw[arrows=-Latex, orange, very thick, densely dotted] (3,4.5) -- (3,4.25);
    \draw[arrows=-Latex, ultra thick, purple] (-0.85,3.85) -- (-0.15,3.15);
    \draw[arrows=-Latex, purple, ultra thick] (0.15,2.85) -- (0.65,2.35);
    \draw[arrows=-Latex, orange, very thick, densely dotted] (2.85,3.85) -- (2.15,3.15);
    \draw[arrows=-Latex, purple, ultra thick] (1.85,2.85) -- (1.35,2.35);
    \draw[arrows=-Latex, purple, ultra thick] (1,1.5) -- (1,1.25);
    \draw[arrows=-Latex, purple, ultra thick] (1,0.75) -- (1,0);
    \draw[arrows=-Latex, orange, very thick, densely dotted] (3.15,3.85) -- (3.85,3.15);
    \draw[arrows=-Latex, teal, loosely dashed, very thick] (4.15,2.85) -- (4.65,2.35);
    \draw[arrows=-Latex, teal, loosely dashed, very thick] (5,1.5) -- (5,1.25);
    \draw[arrows=-Latex, teal, loosely dashed, very thick] (5,0.75) -- (5,0);
    \draw[fill=black!16] (-1,6) circle (0.25cm);
    \draw[fill=white] (-1,5) circle (0.5cm);
    \node[draw=none] at (-1,5) {R};
    \draw[fill=black!40, fill opacity=1] (-1,4) circle (0.25cm);
    \draw[fill=black!16] (3,6) circle (0.25cm);
    \draw[fill=white] (3,5) circle (0.5cm);
    \node[draw=none] at (3,5) {R};
    \draw[fill=orange!80] (3,4) circle (0.25cm);
    \draw[fill=black!16, fill opacity=1] (0,3) circle (0.25cm);
    \draw[fill=black!16] (2,3) circle (0.25cm);
    \draw[fill=white] (1,2) circle (0.5cm);
    \node[draw=none] at (1,2) {C};
    \draw[fill=purple!80] (1,1) circle (0.25cm);
    \draw[fill=black!16] (4,3) circle (0.25cm);
    \draw[fill=white] (5,2) circle (0.5cm);
    \node[draw=none] at (5,2) {R};
    \draw[fill=teal!80] (5,1) circle (0.25cm);
\end{tikzpicture}
	\end{tabular}
}
\caption{%
	(a) Detail of a source computational graph: two upstream reshape operators (R) feed a concatenate operator (C), and the rightmost one feeds a downstream reshape operator (R).
	(b) The algorithm's first step inserts input codec markers (light gray) and output codec markers (dark gray) around each operator.
	(c) The second step assigns each edge a semantics, partitioning the graph into weakly connected subgraphs (edges sharing a color and dash pattern); the third step then identifies each subgraph's absolute dominator (colored node).
	(d) The fourth step merges the leftmost upstream reshape subgraph into the concatenate subgraph; the rightmost upstream reshape subgraph fans out to two subgraphs, and thus cannot be merged without breaking the dominance property.
}
\label{fig:adapt}
\end{figure}
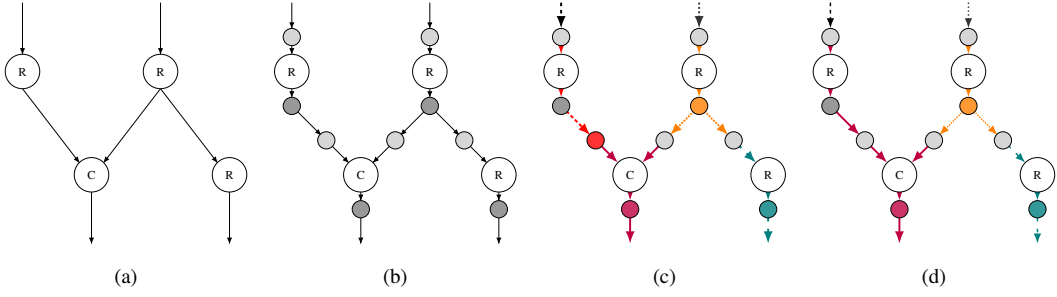

\begin{minipage}{0.60\textwidth}
Algorithm~\ref{algorithm:adapt} decomposes into the steps outlined in Figure~\ref{fig:adapt}:
\begin{enumerate}
    \item[1.] insert codec markers;
    \item[2.] assign semantics to edges;
    \item[3.] induce subgraphs and find absolute dominators;
    \item[4.] coalesce adjacent subgraphs when semantics and absolute dominance can be preserved.
\end{enumerate}
\end{minipage}
\hfill
\begin{minipage}{0.36\textwidth}
\vspace{-\baselineskip}
\begin{algorithm}[H]
\caption{$G \gets adapt(G, \mathfrak{C})$}
\label{algorithm:adapt}
\begin{algorithmic}
	\State $G \gets insert\_markers(G)$
	\State $G \gets assign\_semantics(G, \mathfrak{C})$
	\State $G \gets induce\_subgraphs(G)$
	\State $G \gets coalesce\_subgraphs(G)$
	\\
	\Return $G$
\end{algorithmic}
\end{algorithm}
\end{minipage}

The first step is the only step allowed to make topological changes to the \textbf{graph} $G = (V, E)$ (insert nodes, insert edges, remove edges).
This step preserves the graph's acyclicity \textemdash which guarantees that a topological ordering exists \textemdash and its asymptotic size \textemdash which means that the time complexity of subsequent steps can still be expressed in terms of $G$'s size.

The intuition of codec markers is the same as that of local rules for codec placement.
Consider an operator $f$ whose \textit{source schema} and \textit{target schema} declare different data types, forcing codec placement at its input and output:
\[
	f
	\approx (h^{-1} \circ h) \circ f \circ (g^{-1} \circ g)
	= h^{-1} \circ (h \circ f \circ g^{-1}) \circ g
	= h^{-1} f' g \,.
\]
In plain language:
\begin{itemize}
	\item in the \textit{source view}, the original operator $f$ is approximated by a sequence of a codec operator ($g^{-1} \circ g$), $f$ itself, and another codec operator ($h^{-1} \circ h$);
	\item in the \textit{target view}, the approximation is replaced by the equivalent\footnote{Modulo the algebraic differences between real and digital arithmetic.} sequence of an encoding operator $g$, the operator $f'$ acting on encoded values, and a decoding operator $h^{-1}$.
\end{itemize}

The second step traverses the graph in topological order and paints each edge with a single color, in such a way that same-color edges induce subgraphs that 1) are weakly connected and 2) have a non-empty set of nodes lying on each entry-exit path \textemdash the \textit{absolute dominator candidates}.
Nodes can be transparent to one or more of their input colors, meaning one or more subsets of in- and out-edges must share the same color.
During traversal, each node must enforce color consistency on its transparent sets, potentially recoloring in-edges.
A node's transparent sets are computed in time proportional to the node's degree from a \textbf{commutation table} $\mathfrak{C}$.

The intuition of transparency is that of commutativity with codec operators:
\[
	f \circ (g^{-1} \circ g) = (g^{-1} \circ g) \circ f \,.
\]
This property naturally holds when codecs are lossless (i.e., $g^{-1} \circ g = id$).
However, to commute, $g$ does not need to be the identity; for instance, an element-wise data movement operator $f$ (e.g., a reshape) commutes with an element-wise codec operator $(g^{-1} \circ g)$ (e.g., a fake-quantization).

Note that if recolored in-edges were out-edges of non-markers, allowing a transparent node to recolor its in-edges would open the door to a number of recolorings that can be quadratic in $|E|$ since ancestors might themselves be transparent.
Two examples should suffice as illustrations: a chain of reshape operators and a tree of concatenate operators.
To bound this hazard, the only out-edges allowed to be recolored are the out-edges of input codec markers; the first step guarantees that non-markers are always preceded by input codec markers.

The third step creates the subgraphs spanned by same-color edges and identifies their absolute dominators.
To identify a subgraph's absolute dominator, the step traverses its nodes in topological order while keeping track of each path's last edges in a frontier; whenever a node is traversed whose set of in-edges is a superset of the frontier, that node is added to a set of candidates; the last element of this set is the \textbf{absolute dominator}.
The search does not need a dominator tree.

Same-color subgraphs are best interpreted as subprograms in which, if a codec were materialized on each edge, the function realized by the subgraph would be preserved due to codec idempotence:
\[
	(g^{-1} \circ g) \circ (g^{-1} \circ g)
	= g^{-1} \circ (g \circ g^{-1}) \circ g
	= g^{-1} \circ g \,.
\]
The algorithm thus assumes codecs that satisfy $g \circ g^{-1} = id$ on the encoded space.
To understand the intuition of an absolute dominator, consider a subgraph and the bundle of all its entry-exit paths.
Assume that each path in the bundle passes through a non-empty set of shared ``bottleneck'' nodes: sitting at any such node allows one to observe all the traffic passing through every input and output.
The last node is preferred because the data layout and distribution at that point in the subgraph are closer to those expected by consumer operators.

The fourth step merges adjacent subgraphs when they have compatible colors and the merge preserves the absolute dominance property.

\begin{theorem}
\label{theorem:adapt_correctness}
Algorithm~\ref{algorithm:adapt} is correct.
\begin{itemize}
	\item $G$ is edge-partitioned into same-semantics subgraphs.
	\item Each subgraph 1) is weakly connected, and 2) has an absolute dominator.
\end{itemize}
\end{theorem}

\begin{theorem}
\label{theorem:adapt_complexity}
Algorithm~\ref{algorithm:adapt} terminates.
Its time complexity is $\mathcal{O}(|V| + |E|)$.
\end{theorem}

\section{Evaluation}
\label{section:evaluation}

The measurements are taken with respect to the following cost model: \textbf{storage} sums the size of all weights; \textbf{traffic} sums the size of every weight and feature operand read by every operator \textemdash an object read by multiple operators is counted once per consumer; \textbf{codec items} count the array items crossing materialized codecs.
The cost model does not account for device-specific memory-hierarchy (caching) or datapath (fusion) optimizations.
The cost model is also independent of the codec-learning algorithm and application accuracy, which lie outside this work's scope.

I evaluate the algorithm assuming element-wise linear quantization on four models \textemdash ShuffleNetV2, U-Net, SigLIP, and SmolLM2~\citep{2015Ronneberger, 2018Ma, 2023Zhai, 2025BenAllal} \textemdash and under four compression policies: Z is the 32-bit floating-point baseline; Y uses 8-bit integer matrix-matrix multiplications; X also routes data (concatenate, split) as 8-bit integers; W uses 8-bit integers whenever allowed by the hardware.

\begin{table}
\centering
\begin{tabular}{ l r r r r }
\toprule
				&	  Z	&	  Y	&	  X	&	  W	\\
\midrule
ShuffleNetV2	&   8.8	&   2.3	&   2.3	&   2.3	\\
U-Net 			&  93.3	&  23.4	&  23.4	&  23.4	\\
SigLIP 			& 329.7	&  85.0	&  85.0	&  85.0	\\
SmolLM2 		& 405.1	& 101.4	& 101.4	& 101.4	\\
\bottomrule
\end{tabular}
\caption{Model storage (in MiB) across the four compression policies.}
\label{table:storage}
\end{table}

\begin{table}
\centering
\begin{tabular}{ l r r r r }
\toprule
				&		 Z	&							  Y	&							  X	&									  W	\\
\midrule
ShuffleNetV2	&	   54.8	&	   57.9	   \moreisless{5.7}	&	   58.0	   \moreisless{5.8}	&	   39.4	  \lessismore{28.1}	\\
U-Net			&	  920.4	&	 1039.9	  \moreisless{13.0}	&	 1095.9	  \moreisless{19.1}	&	  759.6	  \lessismore{17.5}	\\
SigLIP			&	 1802.7	&	 2029.8	  \moreisless{12.6}	&	 2047.8	  \moreisless{13.6}	&	 1768.0	   \lessismore{1.9}	\\
SmolLM2			&	  741.1	&	  546.3	  \lessismore{26.3}	&	  541.8	  \lessismore{26.9}	&	  427.2	  \lessismore{42.4}	\\
\bottomrule
\end{tabular}
\caption{%
    Model traffic (in MiB) across the four compression policies.
	Arrows report changes relative to Z, rounded to one decimal place.
    In the absence of hardware support for traffic optimization (e.g., fusing the accumulation and decompression of matrix-matrix multiplication outputs), less aggressive compression policies (Y, X) can increase traffic; only policy W realizes savings on every model.
}
\label{table:traffic}
\end{table}

\begin{table}
\centering
\begin{tabular}{ l r r r r }
\toprule
				&	  Z	&	  Y	&	  X	&	  W	\\
\midrule
ShuffleNetV2	&   0.0	&   4.2	&   6.3	&   4.3	\\
U-Net			&   0.0	&  90.4	& 138.9	&  84.0	\\
SigLIP			&   0.0	& 209.2	& 228.1	& 220.2	\\
SmolLM2			&   0.0	&  47.9	&  50.0	&  47.1	\\
\bottomrule
\end{tabular}
\caption{%
	Number of codec items (in M) across the four compression policies.
	On graphs where chains of uncompressed operators (activations, normalizations) are interleaved with isolated data layout operators (reshape, view), policies prescribing that data layout operators use compressed operands to save traffic can increase the number of codec items with respect to policies prescribing that data layout operators use uncompressed operands.
}
\label{table:codec_items}
\end{table}

Table~\ref{table:storage} shows that storage savings are not affected by the aggressiveness of the compression policy: compressing the weight inputs of matrix-matrix multiplications yields a $75\%$ saving that is consistent across all compression policies.
Table~\ref{table:traffic} shows that policy W reduces total modeled traffic by $2\%$--$42\%$ relative to Z.
However, these totals include the reduction in weight traffic already reflected by Table~\ref{table:storage}.
As a simple decomposition, subtracting the storage reduction once from Z suggests that feature compression reduces modeled traffic by approximately $18\%$ on ShuffleNetV2, $11\%$ on U-Net, and $2\%$ on SmolLM2, while increasing it by approximately $13\%$ on SigLIP.
Thus, under the cost model, aggressive feature compression is not universally beneficial.

\begin{table}[t]
\centering
\begin{tabular}{ l r r r }
\toprule
				&	Local	&				Non-local (w/o)	&				 Non-local (w/)	\\
\midrule
ShuffleNetV2	&	   70.6	&	   39.4	  \lessismore{44.2}	&	   39.4	  \lessismore{44.2}	\\
U-Net			&	 1251.1	&	  767.1	  \lessismore{38.7}	&	  759.6	  \lessismore{39.3}	\\
SigLIP			&	 2311.8	&	 1768.0	  \lessismore{23.5}	&	 1768.0	  \lessismore{23.5}	\\
SmolLM2			&	  643.1 &	  427.2	  \lessismore{33.6}	&	  427.2	  \lessismore{33.6}	\\
\bottomrule
\end{tabular}
\caption{%
    Model traffic (in MiB) across three codec placement algorithms with fixed weight compression.
	Arrows report changes relative to local placement, rounded to one decimal place.
	The local baseline places codecs at the boundaries of every operator that requires encoded operands, independently for each operator; this policy does not track cross-boundary relationships that reveal opportunities for codec parameter sharing, and thus introduces redundant decode-encode pairs.
	The non-local algorithms apply policy W without and with coalescence; coalescence delivers savings for U-Net, which contains several routing and data layout operators.
}
\label{table:local_vs_non_local}
\end{table}

Table~\ref{table:local_vs_non_local} shows that the algorithm reduces model traffic by $24\%$--$44\%$ with respect to local placement.
As a sanity check, I collected exploratory GPU measurements for ShuffleNetV2: the latency drops by $30\%$ from $0.70\,\mathrm{ms} \pm 0.35\,\mathrm{\mu s}$ to $0.49\,\mathrm{ms} \pm 0.53\,\mathrm{\mu s}$ when switching from local placement to non-local placement; the $0.7$ ratio of latency saving to modeled traffic saving is consistent with operand traffic accounting for only part of ShuffleNetV2's execution time.

\section{Discussion}
\label{section:discussion}

This work presents a linear-time algorithm that achieves non-local placement of codec operators from local, per-operator commutation rules and edge-partitions the graph into weakly connected subgraphs, each with an absolute dominator.
The algorithm provides correctness guarantees assuming the correctness of the supplied local codec commutation rules; it does not verify these rules.
The coalescence step greedily attempts to merge adjacent subgraphs while preserving the absolute dominance property; it reduces the number of subgraphs but is not claimed to minimize it.

The evaluation shows that non-local placement reduces operand traffic across all four architectures, even though the policy comparison shows that feature compression itself is not beneficial for every model.
Latency measurements for ShuffleNetV2 align with the predictions of the cost model \textemdash with the caveat that these figures are exploratory and not part of a broad measurement campaign.
The learning of parametric codecs at absolute dominators (calibration, fine-tuning) and application accuracy evaluations are downstream steps, outside this work's scope.

A reflection on the nature of codec commutation rules provides insight into the trade-offs between local and non-local rules for non-local codec placement.
In its simplest form, the semantics attached to each edge is a triple of a unique identifier, a source data type, and a target data type; this definition does not capture whether the parameters of a codec are derived from the parameters of other codecs \textemdash e.g., a dequantization scale obtained as the product of a weight quantization scale and a feature quantization scale, or the block structure of a block-quantized tensor obtained by concatenating two block-quantized tensors.
The simplicity of the algorithm is not inherent, but follows directly from the simplicity of the commutation rules: because a commutation rule simply declares that certain ports share a semantics, assignment reduces to enforcing an equivalence relation over edges, and it is this equivalence that makes non-local placement derivable from local rules in linear time.
Enriching the semantics to include codec parameter transformations would change the assignment problem from locally enforcing an equivalence relationship to non-locally composing transformation rules, potentially reintroducing the reconciliation cost that parameter sharing specifications and pattern catalogs pay through enumeration and arbitration.
Determining which classes of transformation preserve the weak connectedness, absolute dominance, and linear-time guarantees is an open question.

\acksection

I thank Leonardo Lucio Custode, Massimiliano Datres, Martino Dazzi, Johannes Thiele, Francesco Vidaich, and Xin Wang for their feedback on the manuscript.
I thank Evangelos Eleftheriou, Theodoros Rekatsinas, and Bram-Ernst Verhoef for managerial support.

This work has been supported by the European Union through the European High-Performance Computing Joint Undertaking ({EuroHPC JU}) under Grant 101202459 ({DARE-SGA1}).

\clearpage
\bibliographystyle{plainnat}
\bibliography{bibliography}

\clearpage
\appendix

\section{Pseudocode \& proofs}
\label{section:pseudocode_and_proofs}

\subsection*{Notation}

\subsubsection*{Semantics}
In its simplest form, a \textbf{semantics} is a triple $\mathbf{s} = (\mathtt{i}, \mathtt{s}, \mathtt{t})$ where $\mathtt{i}$ is a unique identifier, $\mathtt{s}$ is a \textbf{source} data type, and $\mathtt{t}$ is a \textbf{target} data type.
The intuition of source and target data types is the following: the source data type is the data type of an array in the source model; the target data type is the data type of the array in the compressed model.
Two semantics are said to be \textbf{compatible} if they have the same source and target data types.

\subsubsection*{Port blocks and commutation table}

A \textbf{port block} describes which input- and output-port sets of an operator must share the same semantics.
As an example, consider a binary concatenate operator: both its input ports and its output port must share the same semantics.
Port blocks for all operators are aggregated in a \textbf{commutation table} $\mathfrak{C}$.

\subsubsection*{Graph}
In this work, a \textbf{graph} is a directed acyclic multigraph $G = (V, E)$.

Items $v \in V$ are \textbf{nodes}.
Each node has three attributes:
\begin{itemize}
	\item $.operator$ is the \textbf{operator} represented by $v$;
	\item $.input$ is a dictionary mapping each \textbf{input port} $p$ to its semantics $\mathbf{p}$;
	\item $.output$ is a dictionary mapping each \textbf{output port} $q$ to its semantics $\mathbf{q}$.
\end{itemize}

Items $e \in E$ are \textbf{edges}; an edge is a triple $e = (u, v, (q, p))$, where $u \in V$ is the \textbf{tail}, $v \in V$ is the \textbf{head}, and $(q, p)$ is a \textbf{key} denoting which \textit{output port} $q$ of $u.operator$ writes the object that \textit{input port} $p$ of $v.operator$ reads.
Each edge has two attributes:
\begin{itemize}
	\item $.semantics$ is the semantics $\mathbf{e}$ of the data routed along the edge;
	\item $.subgraph$ is the subgraph $G_{k} \subseteq G$ to which the edge belongs and over which codecs commute.
\end{itemize}

The set $E^{-}(v) = \{ (u, \hat{v}, (q, p)) \in E \,|\, \hat{v} = v \}$ contains the \textbf{in-edges} of node $v$, and $\delta^{-}(v) = |E^{-}(v)|$ is its \textbf{in-degree}.
The set $E^{+}(v) = \{ (\hat{v}, w, (q, p)) \in E \,|\, \hat{v} = v \}$ contains the \textbf{out-edges} of node $v$, and $\delta^{+}(v) = |E^{+}(v)|$ is its \textbf{out-degree}.

The \textbf{absolute dominator} of a weakly connected subgraph $G_{k} \subseteq G$ is the maximal element of the totally ordered set of nodes of $G_{k}$ that are on every entry-exit path.

\subsubsection*{Meta-graph}
A meta-graph $\mathbb{G} = (\mathbb{V}, \mathbb{E})$ is a directed graph.

Items $\mathbb{v}_{k} \in \mathbb{V}$ are \textbf{meta-nodes}.
Each meta-node has one attribute:
\begin{itemize}
	\item $.subgraph$ is a weakly connected subgraph $G_{k} \subseteq G$ induced by same-semantics edges.
\end{itemize}

Items $\mathbb{e} = (\mathbb{v}_{k}, \mathbb{v}_{j}) \in \mathbb{E}$ are \textbf{meta-edges}.
Each meta-edge has one attribute:
\begin{itemize}
	\item $.boundary$ is the set of codec markers that are both exits of $\mathbb{v}_{k}.subgraph$ (the upstream subgraph $G_{k}$) and entries of $\mathbb{v}_{j}.subgraph$ (the downstream subgraph $G_{j}$).
\end{itemize}

The set $\mathbb{E}^{-}(\mathbb{v})$ is the set of \textbf{in-edges} of meta-node $\mathbb{v}$, and $\Delta^{-}(\mathbb{v})$ is its \textbf{in-degree}.
The set $\mathbb{E}^{+}(\mathbb{v})$ is the set of \textbf{out-edges} of meta-node $\mathbb{v}$, and $\Delta^{+}(\mathbb{v})$ is its \textbf{out-degree}.

\subsubsection*{Data structures}

All containers are assumed to be ordered.
The $i$-th item of container $C$ can be accessed with the $C[i]$ notation; positional access and membership tests do constant work ($\mathcal{O}(1)$).
The $copy$ function traverses a container $C$ and returns another container with pointers to the original items; the $reversed$ function traverses a container and returns another container with pointers to the original items, but in reverse order; both functions do work proportional to the size $|C|$ of the source container ($\mathcal{O}(|C|)$).

Dictionaries are assumed to be ordered; given a dictionary $d$, $d.values()$ returns a container of its values.
When a key is not found in a special $defaultdict$ dictionary, the dictionary returns a new instance as prescribed by its constructor's argument.

Set operations are assumed to be performed in place whenever an assignment's left-hand side is an operand on the assignment's right-hand side.
For instance, $B \gets A \cup B$ means that $B$ is updated to include the items in $A$, with work proportional to $A$'s size ($\mathcal{O}(|A|)$).

\subsection*{Adapt}

\subsubsection*{Proof of Theorem~\ref{theorem:adapt_correctness}}

\begin{proof}
The claim follows from Theorem~\ref{theorem:induce_subgraphs_correctness} and Theorem~\ref{theorem:coalesce_subgraphs_correctness}.
\end{proof}

\subsubsection*{Proof of Theorem~\ref{theorem:adapt_complexity}}

\begin{proof}
The algorithm consists of four steps, with the following time complexities:
\begin{itemize}
	\item $\mathcal{O} (|V| + |E|)$ (see Theorem~\ref{theorem:insert_markers_complexity});
	\item $\mathcal{O} (|V| + |E|)$ (see Theorem~\ref{theorem:assign_semantics_complexity});
	\item $\mathcal{O} (|V| + |E|)$ (see Theorem~\ref{theorem:induce_subgraphs_complexity});
	\item $\mathcal{O} (|V| + |E|)$ (see Theorem~\ref{theorem:coalesce_subgraphs_complexity}).
\end{itemize}
Since the steps are executed sequentially, the time complexity sums to $\mathcal{O}(|V| + |E|)$.
\end{proof}

\subsection*{Insert markers}

\begin{algorithm}
\caption{$G \gets insert\_markers(G)$}
\label{algorithm:insert_markers}
\begin{algorithmic}
	\State $G \gets insert\_mark\_input\_codecs(G)$
	\State $G \gets insert\_mark\_output\_codecs(G)$
	\\
	\Return $G$
\end{algorithmic}
\end{algorithm}

\begin{theorem}
\label{theorem:insert_markers_correctness}

Algorithm~\ref{algorithm:insert_markers} is correct.
\begin{itemize}
	\item After execution, each node $v \in V$ satisfies exactly one of the following:
	\begin{enumerate}
		\item[1.] it is a $mark\_input\_codec$;
		\item[2.] it is a $mark\_output\_codec$;
		\item[3.] it has only $mark\_input\_codec$ predecessors and $mark\_output\_codec$ successors.
	\end{enumerate}
	\item The algorithm is idempotent.
\end{itemize}

\begin{proof}
The first claim of Theorem~\ref{theorem:insert_mark_input_codecs_correctness} guarantees that, after executing Algorithm~\ref{algorithm:insert_mark_input_codecs}, each non-marker has only $mark\_input\_codec$ predecessors.
Thus, Algorithm~\ref{algorithm:insert_mark_output_codecs} cannot place $mark\_output\_codec$ nodes before non-markers.
The first claim of Theorem~\ref{theorem:insert_mark_output_codecs_correctness} guarantees that, after executing Algorithm~\ref{algorithm:insert_mark_output_codecs}, each non-marker has only $mark\_output\_codec$ successors.
This proves the first claim.

The first claim of Theorem~\ref{theorem:insert_mark_input_codecs_correctness} guarantees that Algorithm~\ref{algorithm:insert_mark_output_codecs} cannot break the invariant established by Algorithm~\ref{algorithm:insert_mark_input_codecs}.
Thus, the second claim of Theorem~\ref{theorem:insert_mark_input_codecs_correctness} guarantees that any application of Algorithm~\ref{algorithm:insert_mark_input_codecs} after the first does not edit the graph, and cannot break the invariant established by Algorithm~\ref{algorithm:insert_mark_output_codecs}.
Since the graph remains unchanged, the second claim of Theorem~\ref{theorem:insert_mark_output_codecs_correctness} guarantees that any application of Algorithm~\ref{algorithm:insert_mark_output_codecs} after the first also does not edit the graph.
This proves the second claim.
\end{proof}

\end{theorem}

\begin{theorem}
\label{theorem:insert_markers_size}

Algorithm~\ref{algorithm:insert_markers} produces a multigraph $G'' = (V'', E'')$ with the following properties:
\begin{itemize}
	\item $G''$ is acyclic;
	\item $|V''| = \mathcal{O}(|V| + |E|)$ and $|E''| = \mathcal{O}(|E|)$ \textemdash i.e., $|G''| = \mathcal{O}(|G|)$.
\end{itemize}

\begin{proof}
The first claim follows from the first claims of Theorem~\ref{theorem:insert_mark_input_codecs_size} and Theorem~\ref{theorem:insert_mark_output_codecs_size}.

The second claims of Theorem~\ref{theorem:insert_mark_input_codecs_size} and Theorem~\ref{theorem:insert_mark_output_codecs_size} imply that $|V''| \leq |V'| + |E'| \leq |V| + 3|E|$ and $|E''| \leq 2|E'| \leq 4|E|$.
This proves the second claim.
\end{proof}

\end{theorem}

\begin{theorem}
\label{theorem:insert_markers_complexity}

Algorithm~\ref{algorithm:insert_markers} terminates.
Its time complexity is $\mathcal{O}(|V| + |E|)$.

\begin{proof}
By Theorem~\ref{theorem:insert_mark_input_codecs_complexity}, the first step has complexity $\mathcal{O}(|V| + |E|)$.
By Theorem~\ref{theorem:insert_mark_output_codecs_complexity}, the second step has complexity $\mathcal{O}(|V| + |E|)$.
Since the steps are executed sequentially, the time complexity sums to $\mathcal{O}(|V| + |E|) + \mathcal{O}(|V| + |E|) = \mathcal{O}(|V| + |E|)$.
\end{proof}

\end{theorem}

\begin{algorithm}
\caption{$G \gets insert\_mark\_input\_codecs(G)$}
\label{algorithm:insert_mark_input_codecs}
\begin{algorithmic}[1]
	\State $V \gets copy(G.nodes)$
	\For{$v \in V$}
		\If{$v.operator \notin \{ mark\_input\_codec, mark\_output\_codec \}$}
			\State $E \gets copy(G.in\_edges(v))$
			\For{$e \in E$}
				\State $(u, \_, (\_, \_)) \gets e$
				\If{$\neg (u.operator = mark\_input\_codec)$}
					\State $\_ \gets splice\_mark\_input\_codec(G, e)$
				\EndIf
			\EndFor
		\EndIf
	\EndFor
	\\
	\Return $G$
\end{algorithmic}
\end{algorithm}

\begin{algorithm}
\caption{$n \gets splice\_mark\_input\_codec(G, e)$}
\label{algorithm:splice_mark_input_codec}
\begin{algorithmic}[1]
	\State $(u, v, (q, p)) \gets e$
	\State $( \_, \mathtt{s}, \mathtt{t} ) \gets v.input[p]$
	\State $G.add\_node(n)$
	\State $n.operator \gets mark\_input\_codec$
	\State $n.input \gets \{ 'input': ( generate\_id(), \mathtt{s}, \mathtt{s} ) \}$
	\State $n.output \gets \{ 'output': ( generate\_id(), \mathtt{s}, \mathtt{t} ) \}$
	\State $G.add\_edge(u, n, (q, 'input'))$
	\State $G.add\_edge(n, v, ('output', p))$
	\State $G.remove\_edge(u, v, (q, p))$
	\\
	\Return $n$
\end{algorithmic}
\end{algorithm}

\begin{theorem}
\label{theorem:insert_mark_input_codecs_correctness}

Algorithm~\ref{algorithm:insert_mark_input_codecs} is correct.
\begin{itemize}
    \item Each node $v \in V$ satisfies exactly one of the following:
    	\begin{enumerate}
    		\item[1.] it is a $mark\_input\_codec$;
    		\item[2.] it is a $mark\_output\_codec$;
    		\item[3.] it has only $mark\_input\_codec$ predecessors.
    	\end{enumerate}
    \item The algorithm is idempotent.
\end{itemize}

\begin{proof}
To prove the first claim, I restate it formally.
Partition $V$ according to each node's operator:
\begin{itemize}
	\item $A := \{ v \in V \,|\, v.operator = mark\_input\_codec \}$;
	\item $B := \{ v \in V \,|\, v.operator = mark\_output\_codec \}$;
	\item $C := V \setminus (A \cup B)$.
\end{itemize}
Partition $E$ according to each edge's head operator:
\begin{itemize}
	\item $L := \{ (u, v, (q, p)) \in E \,|\, v \in A \}$;
	\item $M := \{ (u, v, (q, p)) \in E \,|\, v \in B \}$;
	\item $N := \{ (u, v, (q, p)) \in E \,|\, v \in C \}$.
\end{itemize}
Then define $I := \{ (u, v, (q, p)) \in N \,|\, (u, v) \in A \times C \}$.
With this notation, the first claim states that for any $v \in V$ exactly one of the following properties holds:
\begin{itemize}
	\item $v \in A$;
	\item $v \in B$;
	\item $E^{-}(v) \subseteq I$.
\end{itemize}

The outer loop visits each node present in the input graph exactly once.
Consider the visit to a node $v$.
Let $N(v) := E^{-}(v) \cap N$ and $I(v) := E^{-}(v) \cap I$.
\begin{itemize}
	\item If $v \in A \cup B$, the operator guard skips it and $E^{-}(v) \subseteq L \cup M$ is not modified.
	\item If $v \in C$, the inner loop scans $E^{-}(v)$: the tail guard skips any edge $e \in I(v)$, while an edge $e = (u, v, (q, p)) \in N(v) \setminus I(v)$ is removed and replaced, through $splice\_mark\_input\_codec$, by an edge $(u, n, (q, 'input')) \in L$ and an edge $(n, v, ('output', p)) \in I(v)$; after the loop, $I(v) = N(v)$.
\end{itemize}
Note that $splice\_mark\_input\_codec$ is \textit{local}, in that no later visit modifies $E^{-}(v)$, so $I(v) = N(v)$ persists for $v$ until termination.
Since every $v \in C$ is visited, $I = \cup_{v \in V} I(v) = \cup_{v \in V} N(v) = N$.
This proves the first claim.

To prove the second claim, assume that the properties hold for every $v \in V$.
If $v \in A \cup B$, the operator guard skips it: the graph is not edited.
If $v \in C$, then $I(v) = N(v)$ and the tail guard $u.operator \neq mark\_input\_codec$ is hit for every in-edge, preventing any splicing from occurring: the graph is not edited.
The algorithm performs no edit and is thus idempotent.
This proves the second claim.
\end{proof}

\end{theorem}

\begin{theorem}
\label{theorem:insert_mark_input_codecs_size}

Algorithm~\ref{algorithm:insert_mark_input_codecs} produces a graph $G' = (V', E')$ with the following properties:
\begin{itemize}
	\item $G'$ is acyclic;
	\item $|V'| \leq |V| + |E|$ and $|E'| \leq 2|E|$ \textemdash i.e., $|G'| = \mathcal{O}(|G|)$.
\end{itemize}

\begin{proof}
To prove the first claim, it suffices to prove that a splicing cannot introduce a cycle.
Given an edge $e \in E$, let $\Pi(e) = \{ \pi \,|\, e \in \pi \}$ be the set of all paths to which $e$ belongs.
If $e \in E$ is spliced, each path $\pi = (a, \dots, u, v, \dots, b) \in \Pi(e)$ is replaced by $\pi' = (a, \dots, u, n, v, \dots, b)$.
Any path starting at $n$ must pass through $v$ (its only successor), and $v$ cannot reach any ancestor $a$ of $u$ due to the acyclicity of $G$; so $n$ cannot reach any ancestor $a$ of $u$ and cannot close a cycle.
This proves the first claim.

To prove the second claim, it suffices to count.
Each splicing inserts one node, adds two edges, and removes one edge: it thus adds one node and one edge.
Each edge $e \in E$ can be spliced at most once, so the algorithm can insert at most $|E|$ nodes and $|E|$ edges.
This proves the second claim.
\end{proof}

\end{theorem}

\begin{theorem}
\label{theorem:insert_mark_input_codecs_complexity}

Algorithm~\ref{algorithm:insert_mark_input_codecs} terminates.
Its time complexity is $\mathcal{O}(|V| + |E|)$.

\begin{proof}
The first instruction executes $|V|$ operations.
The outer loop executes $|V|$ iterations \textemdash one for each $v \in V$.
Assume that the operator guard is never hit.
After copying $E^{-}(v)$ ($\mathcal{O}(\delta^{-}(v))$), the inner loop executes $\delta^{-}(v)$ iterations \textemdash one for each $e \in E^{-}(v)$.
Assume that the tail guard is never hit.
Each invocation of Algorithm~\ref{algorithm:splice_mark_input_codec} does constant work ($\mathcal{O}(1)$).
The total work is then bounded by
\begin{align*}
	\mathcal{O}(|V|) + \sum_{v \in V} \left( \mathcal{O}(\delta^{-}(v)) + \delta^{-}(v) \mathcal{O}(1) \right)
	&= \mathcal{O}(|V|) + \sum_{v \in V} \mathcal{O}(\delta^{-}(v)) + \mathcal{O}(1) \sum_{v \in V} \delta^{-}(v) \\
	&= \mathcal{O}(|V|) + \mathcal{O}(|E|) + \mathcal{O}(|E|) \\
	&= \mathcal{O}(|V| + |E|) \,.
\end{align*}
\end{proof}

\end{theorem}

\begin{algorithm}
\caption{$G \gets insert\_mark\_output\_codecs(G)$}
\label{algorithm:insert_mark_output_codecs}
\begin{algorithmic}[1]
	\State $V \gets copy(G.nodes)$
	\For{$v \in V$}
		\If{$v.operator \notin \{ mark\_input\_codec, mark\_output\_codec \}$}
			\State $output\_port\_to\_edges \gets defaultdict(\emptyset)$
			\For{$f \in G.out\_edges(v)$}
				\State $(\_, \_, (q, \_)) \gets f$
				\State $F \gets output\_port\_to\_edges[q]$
				\State $F.add(f)$
			\EndFor
			\For{$F \in output\_port\_to\_edges.values()$}
				\State $(\_, w, (\_, \_)) \gets F[0]$
				\If{$\neg ((len(F) = 1) \land (w.operator = mark\_output\_codec))$}
					\State $\_ \gets splice\_mark\_output\_codec(G, F)$
				\EndIf
			\EndFor
		\EndIf
	\EndFor
	\\
	\Return $G$
\end{algorithmic}
\end{algorithm}

\begin{algorithm}
\caption{$n \gets splice\_mark\_output\_codec(G, F)$}
\label{algorithm:splice_mark_output_codec}
\begin{algorithmic}[1]
	\State $(v, w, (q, p)) \gets F[0]$
	\State $( \_, \mathtt{s}, \mathtt{t} ) \gets v.output[q]$
	\State $G.add\_node(n)$
	\State $n.operator \gets mark\_output\_codec$
	\State $n.input \gets \{ 'input': ( generate\_id(), \mathtt{s}, \mathtt{t} ) \}$
	\State $n.output \gets \{ 'output': ( generate\_id(), \mathtt{s}, \mathtt{s} ) \}$
	\State $G.add\_edge(v, n, (q, 'input'))$
	\For{$f \in F$}
		\State $(\_, w, (\_, p)) \gets f$
		\State $G.add\_edge(n, w, ('output', p))$
	\EndFor
	\For{$f \in F$}
		\State $(\_, w, (q, p)) \gets f$
		\State $G.remove\_edge(v, w, (q, p))$
	\EndFor
	\\
	\Return $n$
\end{algorithmic}
\end{algorithm}

\begin{theorem}
\label{theorem:insert_mark_output_codecs_correctness}

Algorithm~\ref{algorithm:insert_mark_output_codecs} is correct.
\begin{itemize}
	\item Each node $v \in V$ satisfies exactly one of the following:
	\begin{enumerate}
		\item[1.] it is a $mark\_input\_codec$;
		\item[2.] it is a $mark\_output\_codec$;
		\item[3.] it has only $mark\_output\_codec$ successors.
	\end{enumerate}
	\item The algorithm is idempotent.
\end{itemize}

\begin{proof}
To prove the first claim, I restate it formally.
Partition $V$ according to each node's operator:
\begin{itemize}
    \item $A = \{ v \in V \,|\, v.operator = mark\_input\_codec \}$;
    \item $B = \{ v \in V \,|\, v.operator = mark\_output\_codec \}$;
    \item $C = V \setminus (A \cup B)$.
\end{itemize}
Partition $E$ according to each edge's tail operator:
\begin{itemize}
    \item $L = \{ (v, w, (q, p)) \in E \,|\, v \in A \}$;
    \item $M = \{ (v, w, (q, p)) \in E \,|\, v \in B \}$;
    \item $N = \{ (v, w, (q, p)) \in E \,|\, v \in C \}$.
\end{itemize}
Then define $O := \ (v, w, (q, p)) \in N \,|\, (v, w) \in C \times B \}$.
With this notation, the first claim states that for any $v \in V$ exactly one of the following properties holds:
\begin{itemize}
	\item $v \in A$;
	\item $v \in B$;
	\item $E^{+}(v) \subseteq O$.
\end{itemize}

The outer loop visits each node present in the input graph exactly once.
Consider the visit to a node $v \in V$.
Let $N(v) := E^{+}(v) \cap N$ and $O(v) := E^{+}(v) \cap O$; also, for each $q \in v.output$, let $E^{+}(v, q) := \{ (v, w, (\hat{q}, p)) \in E^{+}(v) \,|\, \hat{q} = q \}$ and $O(v, q) := E^{+}(v, q) \cap O$.
\begin{itemize}
	\item If $v \in A \cup B$, the operator guard skips it and $E^{+}(v) \subseteq L \cup M$ is not modified.
	\item If $v \in C$, the first inner loop partitions $E^{+}(v)$ into $\mathcal{E}^{+}(v) := \{ E^{+}(v, q) \,|\, q \in v.output \}$.
	The second inner loop scans $\mathcal{E}^{+}(v)$.
	\begin{itemize}
    	\item If $E^{+}(v, q) = O(v, q)$ and $|O(v, q)| = 1$, the head guard skips it.
    	\item Otherwise, its edges are removed and replaced, through $splice\_mark\_output\_codec$, by $\{ (v, n, (q, 'input')) \} = O(v, q)$ (note that $|O(v, q)| = 1$) and $E^{+}(n, 'output') = \{ (n, w, ('output', p)) \,|\, (v, w, (q, p)) \in E^{+}(v, q) \} \subseteq M$.
	\end{itemize}
	After the loop, $O(v) = N(v)$, and every $O(v, q)$ is a singleton.
\end{itemize}
Note that $splice\_mark\_output\_codec$ is \textit{local}, in that no later visit modifies $E^{+}(v)$, so $O(v) = N(v)$ persists for $v$ until termination.
Since every $v \in C$ is visited, $O = \cup_{v \in V} O(v) = \cup_{v \in V} N(v) = N$.
This proves the first claim.

To prove the second claim, assume that the properties hold for every $v \in V$.
If $v \in A \cup B$, the operator guard skips it: the graph is not edited.
If $v \in C$, then $O(v) = N(v)$, every $O(v, q)$ is a singleton, and the head guard is hit, preventing any splicing from occurring: the graph is not edited.
The algorithm performs no edit and is thus idempotent.
This proves the second claim.
\end{proof}

\end{theorem}

\begin{theorem}
\label{theorem:insert_mark_output_codecs_size}

Algorithm~\ref{algorithm:insert_mark_output_codecs} produces a graph $G' = (V', E')$ with the following properties:
\begin{itemize}
	\item $G'$ is acyclic;
	\item $|V'| \leq |V| + |E|$ and $|E'| \leq 2|E|$ \textemdash i.e., $|G'| = \mathcal{O}(|G|)$.
\end{itemize}

\begin{proof}
To prove the first claim, it suffices to prove that a splicing cannot introduce a cycle.
Given an edge $e \in E$, let $\Pi(e) = \{ \pi \,|\, e \in \pi \}$ be the set of all paths to which $e$ belongs.
Let $E(n)$ be the edge set spliced by $n$, let $\Pi(n) = \cup_{e \in E(n)} \Pi(e)$ be the set of paths modified by the splicing action, and let $\Pi'(n) = \{ (a, \dots, v, n, w, \dots, b) \,|\, (a, \dots, v, w, \dots, b) \in \Pi(n) \}$ be the set of modified paths.
Any path starting at $n$ must pass first through some strict descendant $w$ of $v$, and $w$ cannot reach any ancestor $a$ of $v$ due to the acyclicity of $G$; so $n$ cannot reach any ancestor $a$ of $v$ and cannot close a cycle.
This proves the first claim.

To prove the second claim, it suffices to count.
Each splicing on $E^{+}(v, q)$ inserts one node, adds $1 + |E^{+}(v, q)|$ edges, and removes $|E^{+}(v, q)|$ edges: it thus adds one node and one edge.
Each node has at most $\delta^{+}(v)$ out-edges, so the splicings on a node's out-sets insert at most $\delta^{+}(v)$ nodes and $\delta^{+}(v)$ edges.
If each node were to be spliced, the algorithm could insert at most $\sum_{v \in V} \delta^{+}(v) = |E|$ nodes and an equal number of edges.
This proves the second claim.
\end{proof}

\end{theorem}

\begin{theorem}
\label{theorem:insert_mark_output_codecs_complexity}

Algorithm~\ref{algorithm:insert_mark_output_codecs} terminates.
Its time complexity is $\mathcal{O}(|V| + |E|)$.

\begin{proof}
The first instruction executes $|V|$ operations.
The outer loop executes $|V|$ iterations \textemdash one for each $v \in V$.
Assume that the operator guard is never hit.
After constant work ($\mathcal{O}(1)$), the first inner loop executes $\delta^{+}(v)$ iterations \textemdash one for each $e \in E^{+}(v)$.
This loop partitions $E^{+}(v)$ into
\[
	\mathcal{E}^{+}(v) = \left\{ E^{+}(v, q) = \left\{ (v, w, (\hat{q}, p)) \in E^{+}(v) \,|\, \hat{q} = q \right\} \,|\, q \in v.output \right\} \,.
\]
The second inner loop executes $|\mathcal{E}^{+}(v)|$ iterations \textemdash one for each $E^{+}(v, q) \in \mathcal{E}^{+}(v)$.
Assume that the head guard is never hit.
Each invocation of Algorithm~\ref{algorithm:splice_mark_output_codec} does work proportional to $|E^{+}(v, q)|$.
The total work is then bounded by
\begin{align*}
	\lefteqn{
		\mathcal{O}(|V|) + \sum_{v \in V} \left( \mathcal{O}(1) + \mathcal{O}(\delta^{+}(v)) + \sum_{E^{+}(v, q) \in \mathcal{E}^{+}(v)} \mathcal{O}(|E^{+}(v, q)|) \right)
	} \\
	&= \mathcal{O}(|V|) + \sum_{v \in V} \mathcal{O}(1) + \sum_{v \in V} \mathcal{O}(\delta^{+}(v)) + \sum_{v \in V} \sum_{E^{+}(v, q) \in \mathcal{E}^{+}(v)} \mathcal{O}(|E^{+}(v, q)|) \\
	&= \mathcal{O}(|V|) + \mathcal{O}(|V|) + \mathcal{O}(|E|) + \mathcal{O}(|E|) \\
	&= \mathcal{O}(|V| + |E|) \,.
\end{align*}
\end{proof}

\end{theorem}

\subsection*{Assign semantics}

\begin{algorithm}
\caption{$G \gets assign\_semantics(G, \mathfrak{C})$}
\label{algorithm:assign_semantics}
\begin{algorithmic}[1]
	\State $V \gets topological\_sort(G)$
	\For{$v \in V$}
		\If{$v.operator = mark\_input\_codec$}
			\State $assign\_semantics\_mark\_input\_codec(G, v)$
		\ElsIf{$v.operator = mark\_output\_codec$}
			\State $assign\_semantics\_mark\_output\_codec(G, v)$
		\Else
			\State $assign\_semantics\_non\_marker(G, v, \mathfrak{C})$
		\EndIf
	\EndFor
	\\
	\Return $G$
\end{algorithmic}
\end{algorithm}

\begin{algorithm}
\caption{$assign\_semantics\_mark\_input\_codec(G, v)$}
\label{algorithm:assign_semantics_mark_input_codec}
\begin{algorithmic}[1]
	\State $e \gets G.in\_edges(v)[0]$
	\State $\mathbf{e} \gets e.semantics$
	\State $\mathbf{q} \gets v.output['output']$
	\If{$is\_compatible(\mathbf{e}, \mathbf{q})$}
		\State $\mathbf{s} \gets \mathbf{e}$
	\Else
		\State $\mathbf{s} \gets \mathbf{q}$
	\EndIf
	\For{$f \in G.out\_edges(v)$}
		\State $f.semantics \gets \mathbf{s}$
	\EndFor
\end{algorithmic}
\end{algorithm}

\begin{algorithm}
\caption{$assign\_semantics\_mark\_output\_codec(G, v)$}
\label{algorithm:assign_semantics_mark_output_codec}
\begin{algorithmic}[1]
	\State $e \gets G.in\_edges(v)[0]$
	\State $\mathbf{e} \gets e.semantics$
	\State $\mathbf{s} \gets \mathbf{e}$
	\For{$f \in G.out\_edges(v)$}
		\State $f.semantics \gets \mathbf{s}$
	\EndFor
\end{algorithmic}
\end{algorithm}

\begin{algorithm}
\caption{$assign\_semantics\_non\_marker(G, v, \mathfrak{C})$}
\label{algorithm:assign_semantics_non_marker}
\begin{algorithmic}[1]
	\For{$edge\_block \in compute\_edge\_blocks(G, v, \mathfrak{C})$}
		\State $E \gets edge\_block.incoming$
		\State $F \gets edge\_block.outgoing$
		\State $\mathbf{E} \gets \emptyset$
		\For{$e \in E$}
			\State $\mathbf{e} \gets e.semantics$
			\State $\mathbf{E}.add(\mathbf{e})$
		\EndFor
		\State $\mathbf{S} \gets \emptyset$
		\For{$e \in E$}
			\State $(\_, v, (\_, p)) \gets e$
			\State $\mathbf{p} \gets v.input[p]$
			\State $\mathbf{S}.add(\mathbf{p})$
		\EndFor
		\For{$f \in F$}
			\State $(v, \_, (q, \_)) \gets f$
			\State $\mathbf{q} \gets v.output[q]$
			\State $\mathbf{S}.add(\mathbf{q})$
		\EndFor
		\If{$len(\mathbf{E}) = 1$}
			\State $\mathbf{s} \gets \mathbf{E}[0]$
		\Else
			\State $\mathbf{s} \gets \mathbf{S}[0]$
		\EndIf
		\For{$e \in E$}
			\State $e.semantics \gets \mathbf{s}$
		\EndFor
		\For{$f \in F$}
			\State $f.semantics \gets \mathbf{s}$
		\EndFor
	\EndFor
\end{algorithmic}
\end{algorithm}

\begin{algorithm}
\caption{$[ edge\_block ] \gets compute\_edge\_blocks(G, v, \mathfrak{C})$}
\label{algorithm:compute_edge_blocks}
\begin{algorithmic}[1]
	\State $port\_index \gets \mathfrak{C}[v.operator]$
	\State $port\_block\_to\_edge\_block \gets defaultdict(EdgeBlock(incoming \gets \emptyset, outgoing \gets \emptyset))$
	\For{$e \in G.in\_edges(v)$}
		\State $(\_, v, (\_, p)) \gets e$
		\State $port\_block \gets port\_index.input[p]$
		\State $edge\_block \gets port\_block\_to\_edge\_block[port\_block]$
		\State $edge\_block.incoming.add(e)$
	\EndFor
	\For{$f \in G.out\_edges(v)$}
		\State $(v, \_, (q, \_)) \gets f$
		\State $port\_block \gets port\_index.output[q]$
		\State $edge\_block \gets port\_block\_to\_edge\_block[port\_block]$
		\State $edge\_block.outgoing.add(f)$
	\EndFor
	\\
	\Return $port\_block\_to\_edge\_block.values()$
\end{algorithmic}
\end{algorithm}

\begin{theorem}
\label{theorem:assign_semantics_correctness}

Algorithm~\ref{algorithm:assign_semantics} is correct.
\begin{itemize}
	\item Each edge $e \in E$ has a unique semantics.
	\item Each semantics-spanned subgraph 1) is weakly connected and 2) has an absolute dominator candidate.
\end{itemize}

\begin{proof}
To prove the first claim, I restate it formally.
Partition $V$ according to each node's visit state:
\begin{itemize}
	\item $A$: the node has not been visited;
	\item $B$: the node has been visited.
\end{itemize}
Partition $E$ according to the visit states of each edge's tail and head:
\begin{itemize}
	\item $L = \{ (u, v, (q, p)) \,|\, u, v \in A \}$;
	\item $M = \{ (u, v, (q, p)) \,|\, (u, v) \in B \times A \}$;
	\item $N = \{ (u, v, (q, p)) \,|\, u, v \in B \}$.
\end{itemize}
With this notation, the first claim states that after execution $N = E$, and that each edge $e \in N$ has a single color.

The outer loop visits each node exactly once, in topological order.
Thus, each edge $e = (u, v, (q, p))$ is visited exactly twice:
\begin{itemize}
	\item when its tail $u$ is visited, $e$ moves from $L$ to $M$ and is assigned a single color;
	\item when its head $v$ is visited, $e$ moves from $M$ to $N$ and may be assigned a new single color.
\end{itemize}
Once an edge is in $N$, it has a single color and cannot be visited again, so the color persists until termination.
Since every $e \in E$ is visited exactly twice, after execution $N = E$.
This proves the first claim.

To prove the second claim, I proceed by induction.
Let
\[
	\mathcal{E}(v) = \{ E_{1}(v), \dots, E_{J(v)}(v) \}
\]
be a partition of $E(v) = E^{-}(v) \cup E^{+}(v)$ into non-empty sets.
Let $E^{-}_{j}(v) := E^{-}(v) \cap E_{j}(v)$ and $E^{+}_{j}(v) := E^{+}(v) \cap E_{j}(v)$; also, let $V^{-}_{j}(v) := \{ u \,|\, (u, v, (q, p)) \in E^{-}_{j}(v) \}$ and $V^{+}_{j}(v) := \{ w \,|\, (v, w, (q, p)) \in E^{+}_{j}(v) \}$.
Each $E_{j}(v)$ falls into one of five cases:
\begin{enumerate}
	\item[1.] $E^{-}_{j}(v) = \emptyset$, $E^{+}_{j}(v) \neq \emptyset$ is colored with a new color;
	\item[2.] $E^{-}_{j}(v) \neq \emptyset$ is recolored, $E^{+}_{j}(v) \neq \emptyset$ is colored with the (new) upstream color;
	\item[3.] $E^{-}_{j}(v) \neq \emptyset$ is recolored, $E^{+}_{j}(v) = \emptyset$;
	\item[4.] $E^{-}_{j}(v) \neq \emptyset$ is not recolored, $E^{+}_{j}(v) \neq \emptyset$ is colored with the upstream color;
	\item[5.] $E^{-}_{j}(v) \neq \emptyset$ is not recolored, $E^{+}_{j}(v) = \emptyset$.
\end{enumerate}

Cases 1, 2, and 3 fall into the base step.
The edges in $E_{j}(v)$ are painted with a fresh color that is shared by no other visited edge $e \in E$; thus, $G_{j}$ is the subgraph spanned by $E_{j}(v)$.
All edges of $E_{j}(v)$ are incident to $v$, so all predecessors and successors are pairwise weakly connected through $v$ (first clause); every entry-exit path of the star-shaped $G_{j}$ passes through $v$ (second clause).
Note that recoloring edges does not break weak connectedness of upstream subgraphs, since recolored edges are always terminal \textemdash never internal.

Cases 4 and 5 fall into the inductive step.
Let $G_{j}$ be the subgraph spanned by the edges with the same color as $E_{j}(v)$ before visiting $v$, and let $G_{j}' \supseteq G_{j}$ be its supergraph obtained by inserting the edges in $E^{+}_{j}(v)$ while visiting $v$.
\begin{enumerate}
	\item[4.] $G_{j}' \neq G_{j}$. By the inductive hypothesis, $G_{j}$ is weakly connected; since every edge $e \in E^{+}_{j}(v)$ is incident to $v$ and $v \in V_{j}$, $G_{j}'$ is also weakly connected (first clause).
		By the inductive hypothesis, given the set $\Pi$ of all entry-exit paths in $G_{j}$, there is an absolute dominator candidate $v^{*}_{j}$ lying on each $\pi \in \Pi$; partition $\Pi$ into $\Pi(v) := \{ (u, \dots, \hat{v}) \in \Pi \,|\, \hat{v} = v \}$ and $\Pi \setminus \Pi(v)$: $v^{*}_{j}$ still lies on every $\pi \in \Pi \setminus \Pi(v)$, but also on every one-edge extension $\pi' \in \Pi'(v) := \{ (u, \dots, v, w) \,|\, ((u, \dots, v) \in \Pi(v)) \land (w \in V^{+}_{j}(v)) \}$; this proves the existence of an absolute dominator in $G_{j}'$ (second clause).
	\item[5.] $G_{j}' = G_{j}$. By the inductive hypothesis, $G_{j}$ is weakly connected (first clause) and has an absolute dominator candidate (second clause).
\end{enumerate}

To conclude the proof, I must map each statement in the algorithm to one of the five cases.
\begin{itemize}
	\item Callee Algorithm~\ref{algorithm:assign_semantics_mark_input_codec} can induce either of two partitions: $\mathcal{E}(v) = \{ E_{1}(v) = E^{-}(v), E_{2}(v) = E^{+}(v) \}$ (cases 1 and 5) or $\mathcal{E}(v) = \{ E_{1}(v) = E(v) \}$ (case 4).
	\item Callee Algorithm~\ref{algorithm:assign_semantics_mark_output_codec} can induce only one partition: $\mathcal{E}(v) = \{ E_{1}(v) = E(v) \}$ (case 4).
	\item Callee Algorithm~\ref{algorithm:assign_semantics_non_marker} can induce any partition (cases 1 to 5).
\end{itemize}
This proves the second claim.
\end{proof}

\end{theorem}

\begin{theorem}
\label{theorem:assign_semantics_complexity}

Algorithm~\ref{algorithm:assign_semantics} terminates.
Its time complexity is $\mathcal{O}(|V| + |E|)$.

\begin{proof}
The topological sort does work $\mathcal{O}(|V| + |E|)$.
The loop executes $|V|$ iterations \textemdash one for each $v \in V$.
Each iteration does constant work ($\mathcal{O}(1)$), then invokes exactly one of three callees, meaning its time complexity is bounded by the maximum of the callees' time complexities.
\begin{itemize}
	\item Algorithm~\ref{algorithm:assign_semantics_mark_input_codec} does work $\mathcal{O}(|E(v)|)$;
	\item Algorithm~\ref{algorithm:assign_semantics_mark_output_codec} does work $\mathcal{O}(|E(v)|)$;
	\item Algorithm~\ref{algorithm:assign_semantics_non_marker} first partitions $E(v)$ into $\mathcal{E}(v) = \{ E_{1}(v), \dots, E_{J(v)}(v) \}$ (Algorithm~\ref{algorithm:compute_edge_blocks}) doing work $\mathcal{O}(|E(v)|)$, then iterates over $\mathcal{E}(v)$ \textemdash one iteration for each $E_{j}(v) \in \mathcal{E}(v)$.
	Each iteration does work $\mathcal{O}(|E_{j}(v)|)$.
	The total work of the callee is bounded by
	\[
		\mathcal{O}(|E(v)|) + \sum_{E_{j}(v) \in \mathcal{E}(v)} \mathcal{O}(|E_{j}(v)|) = \mathcal{O}(|E(v)|) + \mathcal{O}(|E(v)|) = \mathcal{O}(|E(v)|) \,.
	\]
\end{itemize}
The total work is bounded by
\begin{align*}
	\mathcal{O}(|V| + |E|) + \sum_{v \in V} \left( \mathcal{O}(1) + \mathcal{O}(|E(v)|) \right)
	&= \mathcal{O}(|V| + |E|) + \sum_{v \in V} \mathcal{O}(1) + \sum_{v \in V} \mathcal{O}(|E(v)|) \\
	&= \mathcal{O}(|V| + |E|) + \mathcal{O}(|V|) + \mathcal{O}(|E|) \\
	&= \mathcal{O}(|V| + |E|) \,.
\end{align*}
\end{proof}

\end{theorem}

\subsection*{Induce subgraphs}

\begin{algorithm}
\caption{$G \gets induce\_subgraphs(G)$}
\label{algorithm:induce_subgraphs}
\begin{algorithmic}[1]
	\State $semantics\_to\_subgraph \gets dict()$
	\For{$v \in G.nodes$}
		\For{$e \in G.in\_edges(v)$}
			\State $(u, v, (q, p)) \gets e$
			\State $\mathbf{e} \gets e.semantics$
			\If{$\neg(\mathbf{e} \in semantics\_to\_subgraph.keys())$}
				\State $H \gets Subgraph($
					\Statex \hskip 6em $semantics \gets \mathbf{e},$
					\Statex \hskip 6em $nodes \gets \emptyset,$
					\Statex \hskip 6em $edges \gets \emptyset,$
					\Statex \hskip 6em $entries \gets \emptyset,$
					\Statex \hskip 6em $exits \gets \emptyset,$
					\Statex \hskip 6em $dominators \gets \emptyset,$
					\Statex \hskip 6em $entrying \gets \emptyset,$
					\Statex \hskip 6em $exiting \gets \emptyset,$
					\Statex \hskip 6em $dominating \gets \emptyset,$
				\Statex \hskip 4.5em $)$
				\State $semantics\_to\_subgraph[\mathbf{e}] \gets H$
			\EndIf
			\State $H \gets semantics\_to\_subgraph[\mathbf{e}]$
			\State $H.add\_node(u)$
			\State $H.add\_node(v)$
			\State $H.add\_edge(u, v, (q, p))$
		\EndFor
	\EndFor
	\For{$H \in semantics\_to\_subgraph.values()$}
		\State $H \gets analyze\_subgraph(H)$
	\EndFor
	\For{$H \in semantics\_to\_subgraph.values()$}
		\For{$e \in H.edges$}
			\State $e.subgraph \gets H$
		\EndFor
	\EndFor
	\\
	\Return $G$
\end{algorithmic}
\end{algorithm}

\begin{algorithm}
\caption{$H \gets analyze\_subgraph(H)$}
\label{algorithm:analyze_subgraph}
\begin{algorithmic}[1]
	\For{$v \in H.nodes$}
		\If{$H.in\_degree(v) = 0$}
			\State $O \gets copy(H.out\_edges(v))$
			\State $H.entries.add(v)$
			\State $H.entrying \gets H.entrying \cup O$
		\EndIf
	\EndFor
	\For{$v \in H.nodes$}
		\If{$H.out\_degree(v) = 0$}
			\State $I \gets copy(H.in\_edges(v))$
			\State $H.exits.add(v)$
			\State $H.exiting \gets H.exiting \cup I$
		\EndIf
	\EndFor
	\State $V \gets topological\_sort(H)$
	\If{$len(H.entries) = 1$}
		\State $H.dominators \gets copy(H.entries)$
		\State $H.dominating \gets \emptyset$
	\EndIf
	\State $F \gets H.entrying$
	\For{$v \in V$}
		\State $I \gets copy(H.in\_edges(v))$
		\State $O \gets copy(H.out\_edges(v))$
		\If{$I \supseteq F$}
			\State $H.dominators \gets \{ v \}$
			\State $H.dominating \gets I$
			\State $F \gets O$
		\ElsIf{$O \neq \emptyset$}
			\State $F \gets F \setminus I$
			\State $F \gets F \cup O$
		\EndIf
	\EndFor
	\\
	\Return $H$
\end{algorithmic}
\end{algorithm}

\begin{theorem}
\label{theorem:induce_subgraphs_correctness}

Algorithm~\ref{algorithm:induce_subgraphs} is correct.
\begin{itemize}
	\item $G$ is edge-partitioned into $\mathcal{G} = \{ G_{1}, \dots, G_{K} \}$.
	\item Each subgraph $G_{k}$ 1) is weakly connected and 2) has a unique absolute dominator $v^{*}_{k}$.
\end{itemize}

\begin{proof}
The first claim follows from the first claim of Theorem~\ref{theorem:assign_semantics_correctness}.

The first clause of the second claim follows from the second claim of Theorem~\ref{theorem:assign_semantics_correctness}.
The second clause of the second claim of Theorem~\ref{theorem:assign_semantics_correctness} guarantees existence of an absolute dominator for each subgraph; thus, to prove the second clause I must prove uniqueness.
Consider a subgraph $G_{k}$.
$G$ is acyclic and $G_{k} \subseteq G$, so $G_{k}$ is acyclic and has a topological ordering.
Let $V^{*}_{k}$ be the set of nodes lying on each entry-exit path in $G_{k}$; any two nodes $v^{*}, w^{*} \in V^{*}_{k}$ are comparable along any entry-exit path and consistently ordered by every topological ordering, so the items in $V^{*}_{k}$ form a chain.
Each chain is a totally ordered set, and since it is finite it has a unique maximal element: the absolute dominator $v^{*}_{k}$.
This proves the second clause of the second claim.

The absolute dominator $v^{*}_{k}$ of $G_{k}$ is identified by traversing $V_{k}$ in topological order while updating a frontier $F_{k} \subseteq E_{k}$.
$F_{k}$ is initialized as $F_{k} = \{ (u, v, (q, p)) \in E_{k} \,|\, \delta^{-}_{k}(u) = 0 \}$.
Consider the visit of a node $v \in V_{k}$ such that $\delta^{-}_{k}(v) > 0$.
At the start of the visit, $F_{k}$ contains the terminal edges of every path that is incident to any node $u$ appearing before $v$ in the topological ordering.
Also, the invariant $F_{k} \supseteq E^{-}_{k}(v)$ holds.
The update of $F_{k}$ falls into one of three cases:
\begin{itemize}
	\item if $F_{k} \subseteq E^{-}_{k}(v)$, all tracked paths meet at $v$, so $v$ is an absolute dominator candidate; the frontier is emptied and replaced with $E^{+}_{k}(v)$;
	\item if $F_{k} \supset E^{-}_{k}(v)$ and $E^{+}_{k}(v) \neq \emptyset$, the edges in $E^{-}_{k}(v)$ are replaced by those in $E^{+}_{k}(v)$;
	\item if $F_{k} \supset E^{-}_{k}(v)$ and $E^{+}_{k}(v) = \emptyset$, then the edges in $E^{-}_{k}(v)$ cannot exit $F_{k}$ and no node $w$ appearing after $v$ in the topological ordering can be an absolute dominator candidate.
\end{itemize}
Since candidates are discovered in topological order, the last recorded candidate is the maximal one.
\end{proof}

\end{theorem}

\begin{theorem}
\label{theorem:induce_subgraphs_size}

Algorithm~\ref{algorithm:induce_subgraphs} produces an edge-partition $\mathcal{G} = \{ G_{1}, \dots, G_{K} \}$ of $G$ with the following properties:
\begin{itemize}
	\item each $G_{k} \in \mathcal{G}$ is acyclic;
	\item $\sum_{G_{k} \in \mathcal{G}} |V_{k}| \leq 2|E|$ and $\sum_{G_{k} \in \mathcal{G}} |E_{k}| = |E|$ \textemdash i.e., $\mathcal{O}(|\cup_{G_{k} \in \mathcal{G}} G_{k}|) = \mathcal{O}(|G|)$.
\end{itemize}

\begin{proof}
Since $G$ is acyclic and $G_{k} \subseteq G$, $G_{k}$ is acyclic; this holds for each $k = 1, \dots, K$.
This proves the first claim.

To prove the second claim, it suffices to count.
Note that each node $v \in V$ belongs to at most $\delta^{-}(v) + \delta^{+}(v)$ subgraphs; the total number of nodes across the subgraphs $G_{k}$ is thus bounded by $\sum_{v \in V} (\delta^{-}(v) + \delta^{+}(v)) = 2|E|$.
Each edge $e \in E$ belongs to exactly one subgraph; the total number of edges across the subgraphs $G_{k}$ is thus $\sum_{G_{k} \in \mathcal{G}} |E_{k}| = |E|$.
This proves the second claim.
\end{proof}

\end{theorem}

\begin{theorem}
\label{theorem:induce_subgraphs_complexity}

Algorithm~\ref{algorithm:induce_subgraphs} terminates.
Its time complexity is $\mathcal{O}(|V| + |E|)$.

\begin{proof}
The first loop (subgraph creation step) executes $|V|$ iterations \textemdash one for each $v \in V$.
Each iteration does some constant work ($\mathcal{O}(1)$) before its inner loop executes $\delta^{-}(v)$ iterations \textemdash one for each in-edge to $v$.
Each iteration does constant work ($\mathcal{O}(1)$).
The total work of the subgraph creation step is
\[
	\sum_{v \in V} \left( \mathcal{O}(1) + \sum_{e \in E^{-}(v)} \mathcal{O}(1) \right) = \sum_{v \in V} \mathcal{O}(1) + \sum_{v \in V} \mathcal{O}(|E^{-}(v)|) = \mathcal{O}(|V|) + \mathcal{O}(|E|) = \mathcal{O}(|V| + |E|) \,.
\]

The second loop (subgraph analysis step) executes $|\mathcal{G}|$ iterations \textemdash one for each $G_{k}$.
Each iteration does constant work ($\mathcal{O}(1)$), then invokes Algorithm~\ref{algorithm:analyze_subgraph}, which decomposes into four steps.
\begin{itemize}
	\item The first loop does work bounded by $\sum_{v \in V_{k}} \left( \mathcal{O}(1) + \mathcal{O}(|E^{+}_{k}(v)|) \right) = \mathcal{O}(|V_{k}|) + \mathcal{O}(|E_{k}|) = \mathcal{O}(|V_{k}| + |E_{k}|)$.
	\item The second loop does work bounded by $\sum_{v \in V_{k}} \left( \mathcal{O}(1) + \mathcal{O}(|E^{-}_{k}(v)|) \right) = \mathcal{O}(|V_{k}|) + \mathcal{O}(|E_{k}|) = \mathcal{O}(|V_{k}| + |E_{k}|)$.
	\item The topological sort step does work $\mathcal{O}(|V_{k}| + |E_{k}|)$.
	\item The third loop (absolute dominator search step) does work $\sum_{v \in V_{k}} \left( \mathcal{O}(1) + \mathcal{O}(|E^{-}_{k}(v)|) + \mathcal{O}(|E^{+}_{k}(v)|) \right) = \sum_{v \in V_{k}} \mathcal{O}(1) + \sum_{v \in V_{k}} \mathcal{O}(|E^{-}_{k}(v)|) + \sum_{v \in V_{k}} \mathcal{O}(|E^{+}_{k}(v)|) = \mathcal{O}(|V_{k}|) + \mathcal{O}(|E_{k}|) + \mathcal{O}(|E_{k}|) = \mathcal{O}(|V_{k}| + |E_{k}|)$.
\end{itemize}
The total work of the subgraph analysis step is bounded by
\begin{align*}
	\lefteqn{
		\sum_{G_{k} \in \mathcal{G}} \left( \mathcal{O}(1) + \mathcal{O}(|V_{k}| + |E_{k}|) + \mathcal{O}(|V_{k}| + |E_{k}|) + \mathcal{O}(|V_{k}| + |E_{k}|) + \mathcal{O}(|V_{k}| + |E_{k}|) \right)
	} \\
	&= \sum_{G_{k} \in \mathcal{G}} \mathcal{O}(1) + \sum_{G_{k} \in \mathcal{G}} \mathcal{O}(|V_{k}| + |E_{k}|) + \sum_{G_{k} \in \mathcal{G}} \mathcal{O}(|V_{k}| + |E_{k}|) + \sum_{G_{k} \in \mathcal{G}} \mathcal{O}(|V_{k}| + |E_{k}|) + \sum_{G_{k} \in \mathcal{G}} \mathcal{O}(|V_{k}| + |E_{k}|) \\
	&= \mathcal{O}(|E|) + \mathcal{O}(|V| + |E|) + \mathcal{O}(|V| + |E|) + \mathcal{O}(|V| + |E|) + \mathcal{O}(|V| + |E|) \\
	&= \mathcal{O}(|V| + |E|) \,,
\end{align*}
where the second-to-last step follows from the fact that $|\mathcal{G}| \leq |E|$ and Theorem~\ref{theorem:induce_subgraphs_size}.

The third loop (attribute assignment step) executes $|\mathcal{G}|$ iterations \textemdash one for each $G_{k}$.
Each iteration does some constant work ($\mathcal{O}(1)$) before the inner loop executes $|E_{k}|$ iterations \textemdash one for each $e \in E_{k}$.
Each iteration does constant work ($\mathcal{O}(1)$).
The total work of the attribute assignment step is bounded by $\sum_{G_{k} \in \mathcal{G}} \left( \mathcal{O}(1) + \sum_{e \in E_{k}} \mathcal{O}(1) \right) = \sum_{G_{k} \in \mathcal{G}} \mathcal{O}(1) + \sum_{G_{k} \in \mathcal{G}} \mathcal{O}(|E_{k}|) = \mathcal{O}(|E|) + \mathcal{O}(|E|) = \mathcal{O}(|E|)$.

The steps are executed sequentially, so their complexities sum to
\[
	\mathcal{O}(|V| + |E|) + \mathcal{O}(|V| + |E|) + \mathcal{O}(|E|) = \mathcal{O}(|V| + |E|) \,.
\]
\end{proof}

\end{theorem}

\subsection*{Coalesce subgraphs}

\begin{algorithm}
\caption{$G \gets coalesce\_subgraphs(G)$}
\label{algorithm:coalesce_subgraphs}
\begin{algorithmic}[1]
	\State $\mathbb{G} \gets create\_meta\_graph(G)$
	\State $\mathbb{V} \gets reversed(topological\_sort(\mathbb{G}))$
	\For{$\mathbb{v} \in \mathbb{V}$}
		\State $S \gets \mathbb{v}.subgraph$
		\State $\mathbb{M} \gets list()$
		\For{$\mathbb{f} \in \mathbb{G}.out\_edges(\mathbb{v})$}
			\State $(\_, \mathbb{w}) \gets \mathbb{f}$
			\State $T \gets \mathbb{w}.subgraph$
			\State $B \gets \mathbb{f}.boundary$
			\If{$is\_compatible(S.semantics, T.semantics)$}
				\If{$len(B) = len(S.exits)$}
					\State $\mathbb{M}.append(\mathbb{w})$
				\EndIf
			\EndIf
		\EndFor
		\If{$len(\mathbb{M}) = 1$}
			\State $\mathbb{w} \gets \mathbb{M}[0]$
			\State $T \gets \mathbb{w}.subgraph$
			\For{$v \in (S.nodes \setminus S.entries)$}
				\If{$v \in T.entries$}
					\State $T.entries \gets T.entries \setminus \{ v \}$
					\State $T.entrying \gets T.entrying \setminus T.out\_edges(v)$
				\EndIf
			\EndFor
			\State $T.entries \gets T.entries \cup (S.entries \setminus T.entries)$
			\State $T.entrying \gets T.entrying \cup (S.entrying \setminus T.entrying)$
			\State $T.nodes \gets T.nodes \cup S.nodes$
			\State $T.edges \gets T.edges \cup S.edges$
			\If{$T.dominating = \emptyset$}
				\State $T.dominating \gets copy(S.exiting)$
			\EndIf
			\For{$e \in S.edges$}
				\State $e.semantics \gets T.semantics$
				\State $e.subgraph \gets T$
			\EndFor
			\For{$\mathbb{e} \in \mathbb{G}.in\_edges(\mathbb{v})$}
				\State $(\mathbb{u}, \mathbb{v}) \gets \mathbb{e}$
				\State $\mathbb{d} \gets (\mathbb{u}, \mathbb{w})$
				\If{$\neg(\mathbb{d} \in \mathbb{E})$}
					\State $\mathbb{G}.add\_edge(\mathbb{u}, \mathbb{w})$
					\State $\mathbb{d}.boundary \gets \emptyset$
				\EndIf
				\State $\mathbb{d}.boundary \gets \mathbb{d}.boundary \cup \mathbb{e}.boundary$
			\EndFor
			\State $\mathbb{G}.remove\_edge(\mathbb{v}, \mathbb{w})$
			\State $\mathbb{G}.remove\_node(\mathbb{v})$
		\EndIf
	\EndFor
	\\
	\Return $G$
\end{algorithmic}
\end{algorithm}

\begin{theorem}
\label{theorem:coalesce_subgraphs_correctness}

Algorithm~\ref{algorithm:coalesce_subgraphs} is correct.
\begin{itemize}
	\item $G$ is edge-partitioned into $\mathcal{G}' = \{ G_{1}, \dots, G_{K'} \}$.
	\item Each subgraph $G_{k}$ 1) is weakly connected and 2) has a unique absolute dominator $v^{*}_{k}$.
\end{itemize}

\begin{proof}
It suffices to prove that the claims hold after a single merge.

Let $G_{k, j} := G_{k} \cup G_{j}$ be the merged subgraph.
Each edge $e \in E$ belongs either to some $E_{i}$ with $i \notin \{ k, j \}$ or to $E_{k, j} := E_{k} \cup E_{j}$.
This proves the first claim.

$G_{k}$ and $G_{j}$ are weakly connected and share at least one node: since they are merged, $\emptyset \neq B_{k, j} \subseteq V_{k} \cap V_{j}$.
Thus, $G_{k, j}$ is weakly connected.
This proves the first clause of the second claim.

Let $\Pi_{k}$, $\Pi_{j}$, and $\Pi_{k, j}$ be the sets of all entry-exit paths of $G_{k}$, $G_{j}$, and $G_{k, j}$, respectively.
Let $V^{*}_{k, j}$ be the set of nodes lying on every path in $\Pi_{k, j}$.
Note that $B_{k, j} \subseteq G_{j}.entries$ and $G_{k, j}.exits = G_{j}.exits$.
Any node $v \in V_{k}$ lies on some path in $\Pi_{k}$, each such path ends at a node in $B_{k, j}$, and each entry-exit path in $\Pi_{j}$ traverses $v^{*}_{j}$; thus, any node $v \in V_{k}$ is an ancestor of $v^{*}_{j}$.
Any path $(u, \dots, v, \dots, w) \in \Pi_{k, j}$ is either a path in $\Pi_{j}$ or traverses a node $v \in V_{k} \cap V_{j}$ (which the previous statement guarantees to be an ancestor of $v^{*}_{j}$) such that $(v, \dots, w)$ is the suffix of a path in $\Pi_{j}$, in both cases traversing $v^{*}_{j}$; thus, $v^{*}_{j} \in V^{*}_{k, j}$.
Conversely, any path in $\Pi_{j}$ is either in $\Pi_{k, j}$ or can be prefixed so that the extension is in $\Pi_{k, j}$; thus, $V^{*}_{k, j}$ cannot contain any strict descendant of $v^{*}_{j}$.
This proves the second clause of the second claim.
\end{proof}

\end{theorem}

\begin{theorem}
\label{theorem:coalesce_subgraphs_complexity}

Algorithm~\ref{algorithm:coalesce_subgraphs} terminates.
Its time complexity is $\mathcal{O}(|V| + |E|)$.

\begin{proof}
By Theorem~\ref{theorem:create_meta_graph_complexity}, the first step has time complexity $\mathcal{O}(|V| + |E|)$.

By Theorem~\ref{theorem:create_meta_graph_size}, the topological sort of $\mathbb{G}$ has time complexity $\mathcal{O}(|E|)$.

The outer loop (meta-graph traversal) executes at most $|E|$ iterations \textemdash one for each $\mathbb{v}_{k} \in \mathbb{V}$.
Each iteration does some constant work ($\mathcal{O}(1)$) before proceeding to the search and merge steps; this work sums to $\mathcal{O}(|\mathbb{V}|) = \mathcal{O}(|E|)$ across the meta-graph traversal.

The first inner loop (search step) performs $\Delta^{+}(\mathbb{v}_{k})$ iterations \textemdash one for each meta-edge $\mathbb{e} = (\mathbb{v}_{k}, \mathbb{v}_{j}) \in \mathbb{E}^{+}(\mathbb{v}_{k})$.
Each iteration does constant work ($\mathcal{O}(1)$) (since $B \subseteq G_{k}.exits$, $|B| = |G_{k}.exits|$ is equivalent to $B = G_{k}.exits$).
The total work of the search step across the meta-graph traversal is bounded by
\begin{align*}
	\sum_{\mathbb{v}_{k} \in \mathbb{V}} \left( \mathcal{O}(1) + \sum_{(\mathbb{v}_{k}, \mathbb{v}_{j}) \in \mathbb{E}^{+}(\mathbb{v}_{k})} \mathcal{O}(1) \right)
	&= \sum_{\mathbb{v}_{k} \in \mathbb{V}} \mathcal{O}(1) + \sum_{\mathbb{v}_{k} \in \mathbb{V}} \sum_{(\mathbb{v}_{k}, \mathbb{v}_{j}) \in \mathbb{E}^{+}(\mathbb{v}_{k})} \mathcal{O}(1) \\
	&= \mathcal{O}(|\mathbb{V}|) + \mathcal{O}(|\mathbb{E}|) \\
	&= \mathcal{O}(|E|) + \mathcal{O}(|E|) \\
	&= \mathcal{O}(|E|) \,.
\end{align*}

Assume that the merge guard is never hit.
The merge step decomposes into three steps.
\begin{itemize}
	\item The first block (subgraph update step) does work $\mathcal{O}(|V_{k}| + |E_{k}|)$.
	The loop to subtractively update the entry nodes and entrying edges performs $|V_{k}| - |G_{k}.entries|$ iterations \textemdash one for each $v \in (V_{k} \setminus G_{k}.entries)$.
	Note that once a node $v$ is removed from $G_{j}.entries$ or an edge $e$ is removed from $G_{j}.entrying$, they cannot be added again; thus, across all iterations, each node $v \in V_{j}$ and outgoing edge $e \in E^{+}_{j}(v)$ can be removed at most once, bounding the total work by $\mathcal{O}(|V_{j}| + |E_{j}|)$.
	The work to additively update $G_{j}.entries$ and $G_{j}.entrying$ is bounded by $\mathcal{O}(|V_{k}| + |E_{k}|)$.
	The work to update $V_{j}$ and $E_{j}$ is $\mathcal{O}(|V_{k}| + |E_{k}|)$.
	\item The second inner loop (attribute update step) executes $|E_{k}|$ iterations \textemdash one for each $e \in E_{k}$.
	Each iteration does constant work ($\mathcal{O}(1)$), for a total loop work of $\mathcal{O}(|E_{k}|)$.
	\item The third inner loop (meta-graph update) executes $\Delta^{-}(\mathbb{v}_{k})$ iterations \textemdash one for each meta-edge $\mathbb{e} = (\mathbb{v}_{j}, \mathbb{v}_{k}) \in \mathbb{E}^{-}(\mathbb{v}_{k})$.
	Each iteration does constant work ($\mathcal{O}(1)$) before doing work $\mathcal{O}(|B_{j, k}|)$.
	Theorem~\ref{theorem:create_meta_graph_boundary} guarantees that
	\[
		B_{j_{1}, k} \cap B_{j_{2}, k} = \emptyset
	\]
	when $j_{1} \neq j_{2}$, so the total loop work is
	\[
		\sum_{(\mathbb{v}_{j}, \mathbb{v}_{k}) \in \mathbb{E}^{-}(\mathbb{v}_{k})} \left( \mathcal{O}(1) + \mathcal{O}(|B_{j, k}|) \right) = \mathcal{O}(|\mathbb{E}^{-}(\mathbb{v}_{k})|) + \mathcal{O}(|V_{k}|) \,.
	\]
\end{itemize}
The total work of the merge step across the meta-graph traversal is bounded by
\begin{align*}
	\lefteqn{
		\sum_{\mathbb{v}_{k} \in \mathbb{V}} \left( \mathcal{O}(|V_{k}| + |E_{k}|) + \mathcal{O}(|E_{k}|) + \left( \mathcal{O}(|\mathbb{E}^{-}(\mathbb{v}_{k})|) + \mathcal{O}(|V_{k}|) \right) \right)
	} \\
	&= \sum_{\mathbb{v}_{k} \in \mathbb{V}} \mathcal{O}(|V_{k}| + |E_{k}|) + \sum_{\mathbb{v}_{k} \in \mathbb{V}} \mathcal{O}(|E_{k}|) + \left( \sum_{\mathbb{v}_{k} \in \mathbb{V}} \mathcal{O}(|\mathbb{E}^{-}(\mathbb{v}_{k})|) + \sum_{\mathbb{v}_{k} \in \mathbb{V}} \mathcal{O}(|V_{k}|) \right) \\
	&= \mathcal{O}(|E|) + \mathcal{O}(|E|) + \mathcal{O}(|E|) + \mathcal{O}(|E|) \\
	&= \mathcal{O}(|E|) \,.
\end{align*}
Note that no merge increases the out-degree $\Delta^{+}(\mathbb{u})$ of upstream nodes $\mathbb{u} \in \mathbb{V}$: each removed meta-edge is replaced by at most one meta-edge.
Thus, the complexity of the search step is invariant across the meta-graph traversal.

The total work is bounded by the sum of the steps' time complexities:
\[
	\mathcal{O}(|V| + |E|) + \mathcal{O}(|E|) + \mathcal{O}(|E|) + \mathcal{O}(|E|) + \mathcal{O}(|E|) = \mathcal{O}(|V| + |E|) \,.
\]
\end{proof}

\end{theorem}

\begin{algorithm}
\caption{$\mathbb{G} \gets create\_meta\_graph(G)$}
\label{algorithm:create_meta_graph}
\begin{algorithmic}[1]
	\State $\mathbb{G} \gets DiGraph()$
	\For{$v \in G.nodes$}
		\For{$f \in G.out\_edges(v)$}
			\State $\mathbb{v} \gets f.semantics$
			\State $H \gets f.subgraph$
			\State $\mathbb{G}.add\_node(\mathbb{v})$
			\State $\mathbb{v}.subgraph \gets H$
		\EndFor
		\If{$v.operator \in \{ mark\_input\_codec, mark\_output\_codec \}$}
			\State $e \gets G.in\_edges(v)[0]$
			\For{$f \in G.out\_edges(v)$}
				\State $S \gets e.subgraph$
				\State $T \gets f.subgraph$
				\If{$S.semantics \neq T.semantics$}
					\If{$(v \in S.exits) \land (v \in T.entries)$}
						\State $\mathbb{u} \gets S.semantics$
						\State $\mathbb{v} \gets T.semantics$
						\State $\mathbb{e} \gets (\mathbb{u}, \mathbb{v})$
						\If{$\neg(\mathbb{e} \in \mathbb{E})$}
							\State $\mathbb{G}.add\_edge(\mathbb{e})$
							\State $\mathbb{e}.boundary \gets \emptyset$
						\EndIf
						\State $\mathbb{e}.boundary.add(v)$
					\EndIf
				\EndIf
			\EndFor
		\EndIf
	\EndFor
	\\
	\Return $\mathbb{G}$
\end{algorithmic}
\end{algorithm}

\begin{theorem}
\label{theorem:create_meta_graph_size}

Algorithm~\ref{algorithm:create_meta_graph} is correct.
\begin{itemize}
	\item $\mathbb{G}$ is acyclic.
	\item $|\mathbb{V}| \leq |E|$ and $|\mathbb{E}| \leq |E|$ \textemdash i.e., $|\mathbb{G}| = \mathcal{O}(|G|)$.
\end{itemize}

\begin{proof}
To prove the first claim, I show that no node $\mathbb{v}_{k}$ can close a cycle in $\mathbb{G}$.
Let $\mathbb{v}_{k} \in \mathbb{V}$, $G_{k} = \mathbb{v}_{k}.subgraph$, and $v^{*}_{k} \in V_{k}$ be the absolute dominator of $G_{k}$.
Let $\mathbb{A}(\mathbb{v}_{k})$ and $\mathbb{D}(\mathbb{v}_{k})$ be the sets of ancestors and descendants of $\mathbb{v}_{k}$, respectively; by Theorem~\ref{theorem:create_meta_graph_ancestry}, $\mathbb{v}_{k-1} \in \mathbb{A}(\mathbb{v}_{k})$ only if $v^{*}_{k-1}$ is an ancestor of $v^{*}_{k}$, and $\mathbb{v}_{j} \in \mathbb{D}(\mathbb{v}_{k})$ only if $v^{*}_{j}$ is a descendant of $v^{*}_{k}$.
For any $(\mathbb{v}_{k-1}, \mathbb{v}_{j}) \in \mathbb{A}(\mathbb{v}_{k}) \times \mathbb{D}(\mathbb{v}_{k})$, $v^{*}_{k-1}$ is an ancestor of $v^{*}_{k}$, which is an ancestor of $v^{*}_{j}$.
By acyclicity of $G$, there is no path from $v^{*}_{j}$ to $v^{*}_{k-1}$, so $v^{*}_{k}$ cannot close a cycle in $G$.
Thus, $\mathbb{v}_{k}$ cannot close a cycle in $\mathbb{G}$.
This proves the first claim.

The first inner loop adds at most one meta-node for each out-edge $e \in E^{+}(v)$ (different edges can map to the same meta-node): $|\mathbb{V}| \leq \sum_{v \in V} \delta^{+}(v) = |E|$.
The second inner loop adds at most one meta-edge for each out-edge $e \in E^{+}(v)$: $|\mathbb{E}| \leq \sum_{v \in V} \delta^{+}(v) = |E|$.
This proves the second claim.
\end{proof}

\end{theorem}

\begin{theorem}
\label{theorem:create_meta_graph_ancestry}

Let $\mathbb{v}_{k}, \mathbb{v}_{j} \in \mathbb{V}$ be meta-nodes, and $v^{*}_{k}, v^{*}_{j}$ be the absolute dominators of the subgraphs $G_{k}, G_{j}$.
$\mathbb{v}_{k}$ is an ancestor of $\mathbb{v}_{j}$ only if $v^{*}_{k}$ is an ancestor of $v^{*}_{j}$.

\begin{proof}
Since $\mathbb{v}_{k}$ is an ancestor of $\mathbb{v}_{j}$, there is a path $(\mathbb{v}_{k} = \mathbb{v}_{1}, \dots, \mathbb{v}_{I} = \mathbb{v}_{j})$.
For every $i = 2, \dots, I$, there exists $v_{i} \in G_{i-1}.exits \,\cap\, G_{i}.entries$ and a path
\[
	(v^{*}_{i-1}, \dots, v_{i}, \dots, v^{*}_{i})
\]
connecting the absolute dominators of $G_{i-1}$ and $G_{i}$.
Concatenating these paths over all the $i = 2, \dots, I$ yields a path $(v^{*}_{k} = v^{*}_{1}, \dots, v^{*}_{I} = v^{*}_{j})$.
Thus, $v^{*}_{k}$ is an ancestor of $v^{*}_{j}$.
\end{proof}

\end{theorem}

\begin{theorem}
\label{theorem:create_meta_graph_boundary}

The collection $\mathcal{B} = \{ B_{k, j} \,|\, (\mathbb{v}_{k}, \mathbb{v}_{j}) \in \mathbb{E} \}$ is a partition of a subset of $\{ v \in V \,|\, v.operator \in \{ mark\_input\_codec, mark\_output\_codec \} \}$.

\begin{proof}
Only codec markers can belong to boundaries.
Each codec marker has a single in-edge \textemdash so it belongs to a single upstream graph $G_{k}$.
Algorithm~\ref{algorithm:assign_semantics_mark_input_codec} and Algorithm~\ref{algorithm:assign_semantics_mark_output_codec} paint all out-edges of codec markers with the same color \textemdash so each codec marker belongs to a single downstream graph $G_{j}$.
Thus, each codec marker belongs to either no boundary \textemdash $G_{k} = G_{j}$ \textemdash or one boundary \textemdash $G_{k} \neq G_{j}$.
\end{proof}

\end{theorem}

\begin{theorem}
\label{theorem:create_meta_graph_complexity}

Algorithm~\ref{algorithm:create_meta_graph} terminates.
Its time complexity is $\mathcal{O}(|V| + |E|)$.

\begin{proof}
The outer loop executes $|V|$ iterations \textemdash one for each $v \in V$.
The first inner loop executes $\delta^{+}(v)$ iterations \textemdash one for each $e \in E^{+}(v)$.
Each iteration does constant work ($\mathcal{O}(1)$).
Assume the marker guard is never hit.
The second inner loop executes $\delta^{+}(v)$ iterations \textemdash one for each $e \in E^{+}(v)$.
Assume that the color and boundary guards are never hit.
Each iteration does constant work ($\mathcal{O}(1)$).
The total work is bounded by
\begin{align*}
	\sum_{v \in V} \left( \mathcal{O}(1) + \sum_{e \in E^{+}(v)} \mathcal{O}(1) + \sum_{e \in E^{+}(v)} \mathcal{O}(1) \right)
	&= \mathcal{O}(|V|) + \sum_{v \in V} \sum_{e \in E^{+}(v)} \mathcal{O}(1) + \sum_{v \in V} \sum_{e \in E^{+}(v)} \mathcal{O}(1) \\
	&= \mathcal{O}(|V|) + \mathcal{O}(|E|) + \mathcal{O}(|E|) \\
	&= \mathcal{O}(|V| + |E|) \,.
\end{align*}
\end{proof}

\end{theorem}

\clearpage

\section{Evaluation details}
\label{section:evaluation_details}

The algorithm was implemented in PyTorch (2.14) and evaluated on a MacBook M1 Pro.
The latency measurements summarize 512 runs of a model compiled with TensorRT (11.2) and executed on an NVIDIA H200 GPU.

\begin{table}[h]
\centering
\begin{tabular}{ p{0.45\textwidth} p{0.30\textwidth} c c }
\toprule
Operator														&	Port blocks																					& FP32			& INT8								\\
\midrule
{\tiny\url{aten._native_batch_norm_legit_no_training.default}}	& $\varnothing$																					& \textbf{Z}	&									\\
{\tiny\url{aten.add.Tensor}}									& $\varnothing$																					& \textbf{Z}	&									\\
{\tiny\url{aten.cat.default}}									& $\{(\{\mathtt{tensors[*]}\},\{\mathtt{out}\})\}$												& \textbf{Z}	& \textcolor{orange}{\textbf{X}}	\\
{\tiny\url{aten.clone.default}}									& $\{(\{\mathtt{self}\},\{\mathtt{out}\})\}$													& \textbf{Z}	& \textcolor{purple}{\textbf{W}}	\\
{\tiny\url{aten.constant_pad_nd.default}}						& $\{(\{\mathtt{self}\},\{\mathtt{out}\})\}$													& \textbf{Z}	& \textcolor{purple}{\textbf{W}}	\\
{\tiny\url{aten.convolution.default}}							& $\varnothing$																					& \textbf{Z}	& \textcolor{teal}{\textbf{Y}}		\\
{\tiny\url{aten.embedding.default}}								& $\{(\{\mathtt{weight}\},\{\mathtt{out}\})\}$													& \textbf{Z}	&									\\
{\tiny\url{aten.expand.default}}								& $\{(\{\mathtt{self}\},\{\mathtt{out}\})\}$													& \textbf{Z}	& \textcolor{purple}{\textbf{W}}	\\
{\tiny\url{aten.gelu.default}}									& $\varnothing$																					& \textbf{Z}	&									\\
{\tiny\url{aten.layer_norm.default}}							& $\varnothing$																					& \textbf{Z}	&									\\
{\tiny\url{aten.linear.default}}								& $\varnothing$																					& \textbf{Z}	& \textcolor{teal}{\textbf{Y}}		\\
{\tiny\url{aten.max_pool2d_with_indices.default}}				& $\{(\{\mathtt{self}\},\{\mathtt{out[0]}\})\}$													& \textbf{Z}	& \textcolor{purple}{\textbf{W}}	\\
{\tiny\url{aten.mean.dim}}										& $\varnothing$																					& \textbf{Z}	&									\\
{\tiny\url{aten.mul.Tensor}}									& $\varnothing$																					& \textbf{Z}	&									\\
{\tiny\url{aten.permute.default}}								& $\{(\{\mathtt{self}\},\{\mathtt{out}\})\}$													& \textbf{Z}	& \textcolor{purple}{\textbf{W}}	\\
{\tiny\url{aten.relu.default}}									& $\varnothing$																					& \textbf{Z}	&									\\
{\tiny\url{aten.rms_norm.default}}								& $\varnothing$																					& \textbf{Z}	&									\\
{\tiny\url{aten.scaled_dot_product_attention.default}}			& $\varnothing$																					& \textbf{Z}	& \textcolor{teal}{\textbf{Y}}		\\
{\tiny\url{aten.silu.default}}									& $\varnothing$																					& \textbf{Z}	&									\\
{\tiny\url{aten.slice.Tensor}}									& $\{(\{\mathtt{self}\},\{\mathtt{out}\})\}$													& \textbf{Z}	& \textcolor{purple}{\textbf{W}}	\\
{\tiny\url{aten.split_with_sizes.default}}						& $\{(\{\mathtt{self}\},\{\mathtt{out[*]}\})\}$													& \textbf{Z}	& \textcolor{orange}{\textbf{X}}	\\
{\tiny\url{aten.upsample_nearest2d.vec}}						& $\{(\{\mathtt{input}\},\{\mathtt{out}\})\}$													& \textbf{Z}	& \textcolor{purple}{\textbf{W}}	\\
{\tiny\url{aten.view.default}}									& $\{(\{\mathtt{self}\},\{\mathtt{out}\})\}$													& \textbf{Z}	& \textcolor{purple}{\textbf{W}}	\\
{\tiny\url{xaten.fan_out.Tensor}}								& $\{(\{\mathtt{input}\},\{\mathtt{out}\})\}$													& \textbf{Z}	& \textcolor{orange}{\textbf{X}}	\\
{\tiny\url{xaten.kv_cache.default}}								& {\tiny$\{(\{\mathtt{key}\},\{\mathtt{out[0]}\}),(\{\mathtt{value}\},\{\mathtt{out[1]}\})\}$}	& \textbf{Z}	& \textcolor{orange}{\textbf{X}}	\\
{\tiny\url{xaten.rotary_positional_embeddings.default}}			& $\varnothing$																					& \textbf{Z}	& 									\\
\bottomrule
\end{tabular}
\caption{%
	Commutation table $\mathfrak{C}$ used for evaluation.
	Port blocks (i.e., pairs of input- and output-port sets that must share the same semantics) are shown for each operator: $\varnothing$ indicates that all port blocks are trivial (i.e., they contain only pairs of a singleton and an empty set); \texttt{[*]} denotes a variadic port.
	Letters \textbf{Z}, \textcolor{teal}{\textbf{Y}}, \textcolor{orange}{\textbf{X}}, and \textcolor{purple}{\textbf{W}} indicate the first policy that selects each operator variant.
}
\label{table:operators}
\end{table}

\begin{table}[h]
\centering
\begin{tabular}{ l p{0.42\textwidth} l l }
\toprule
Model			& Repository												& License		& Shape					\\
\midrule
ShuffleNetV2	& {\tiny\url{pytorch/vision}}								& BSD 3-Clause	& $(1, 3, 224, 224)$	\\
U-Net			& {\tiny\url{qubvel-org/segmentation_models.pytorch}}		& MIT			& $(1, 3, 512, 512)$	\\
SigLIP			& {\tiny\url{HuggingFaceTB/SmolVLM2-256M-Video-Instruct}}	& Apache 2.0	& $(1, 3, 512, 512)$	\\
SmolLM2			& {\tiny\url{HuggingFaceTB/SmolVLM2-256M-Video-Instruct}}	& Apache 2.0	& $(1, 128, 576)$		\\
\bottomrule
\end{tabular}
\caption{%
	Model population used for evaluation.
	Models were traced to \texttt{torch.fx.GraphModule} objects using the \texttt{torch.export.export} function; rotary positional embedding (RoPE) and root mean square normalization (RMSNorm) subgraphs were pattern-matched and replaced with standalone operators before hardware adaptation was executed.
}
\label{table:models}
\end{table}

Code is available on \href{https://github.com/axelera-ai/model-compressor}{GitHub}.

\end{document}